\documentclass[journal]{IEEEtran}

\usepackage{cite}                    
\usepackage{xcolor}
\usepackage{graphicx}
\usepackage{amsmath,amssymb}
\usepackage{amsthm}
\theoremstyle{plain}
\newtheorem{lemma}{Lemma}
\newtheorem{proposition}{Proposition}
\theoremstyle{definition}
\newtheorem{assumption}{Assumption}
\theoremstyle{remark}
\newtheorem{remark}{Remark}
\usepackage{booktabs}
\usepackage{multirow}
\usepackage{algorithm}
\usepackage{algpseudocode}
\usepackage{subcaption}
\usepackage{tabularx}
\usepackage{stfloats}            

\usepackage[colorlinks=false,
            pdfborder={0 0 1},
            citebordercolor={0 1 0},
            linkbordercolor={1 0 0},
            urlbordercolor={0 1 0}
           ]{hyperref}

\begin{document}
\bstctlcite{IEEEexample:BSTcontrol}
    \title{Average-Power-Budgeted Underwater Vehicle Control via Constrained Reinforcement Learning}
  \author{Yinuo~Wang, Xiaowen~Tao$^{\ast}$, Yuze~Liu, and John~V.~Ringwood,~\IEEEmembership{Life~Fellow,~IEEE}

  \thanks{Yinuo Wang is with the College of Graduate and Professional Studies, Trine University, Angola, IN 46703, USA. Xiaowen Tao is with the School of Computer Science and Statistics, Trinity College Dublin, Dublin, D02 PN40, Ireland. Yuze Liu is with the School of Engineering, Swinburne University of Technology, Melbourne, VIC 3122, Australia. John V. Ringwood is with the Centre for Ocean Energy Research (COER), Maynooth University, Kildare, W23 F2H6, Ireland.}
  \thanks{Yinuo Wang and Xiaowen Tao contributed equally to this work.}
  \thanks{$^{\ast}$Corresponding author: Xiaowen Tao (taox@tcd.ie).}

}

\markboth{IEEE Transactions on Systems, Man, and Cybernetics, July~2026}%
{Tao \MakeLowercase{\textit{et al.}}: Average-Power-Budgeted Underwater Vehicle Control via Constrained Reinforcement Learning}

\maketitle

\begin{abstract}
Autonomous underwater vehicles increasingly run long-duration missions from a fixed energy budget that propulsion rapidly depletes; on such endurance-limited platforms, lower average thruster power extends mission range. Classical model-based controllers hinge on hard-to-identify hydrodynamic parameters, motivating model-free reinforcement learning (RL), which attains station-keeping and trajectory tracking from interaction; but an RL policy optimizes only what its reward encodes, so rewarding task accuracy alone yields oscillatory, power-hungry actuation. Established remedies add an energy penalty governed by a dimensionless weight; lacking physical units, it cannot specify the average power in advance, does not transfer across vehicles or tasks, must be re-tuned by hand, and can even increase power when badly tuned. This paper instead formulates energy-efficient control as a constrained Markov decision process that bounds average thruster power by an explicit budget, in watts, solved with a proximal policy optimization (PPO)--Lagrangian algorithm whose dual variable adapts online toward that budget. The budget thus acts as a direct, physical knob on average power, applied unchanged across vehicles and tasks without per-setting weight search. Across three vehicles and four tasks in MarineGym, the constrained policy draws the lowest average power in all twelve settings, 14--32\% below a task-only baseline and below an energy-reward baseline everywhere, while remaining smoothest in eleven and preserving task success in most. More broadly, treating the average-power budget as an explicit, adaptive constraint offers a declarable, tuning-free route to energy-efficient control for long-endurance ocean missions. A demonstration video is available at \textbf{\textcolor{blue}{\url{https://youtu.be/j-CF4p_G2Vo}}}.
\end{abstract}


\IEEEpeerreviewmaketitle

\section{Introduction}

\IEEEPARstart{U}{nderwater} robots have become a focal point of marine-robotics research, increasingly deployed for ocean exploration, environmental monitoring, seabed survey, and the inspection of submerged infrastructure~\cite{er2024intelligent,wang2025editorial}. Such platforms, spanning autonomous underwater vehicles (AUVs) and remotely operated vehicles (ROVs), run from a fixed onboard energy budget that propulsion rapidly depletes~\cite{fossen2011handbook}: a deployed vehicle cannot recharge, and the dense, viscous medium makes propulsion the dominant energy draw. For an endurance-limited platform, a controller that completes its task while drawing less average power directly extends mission range and dive duration, so the control problem is two-sided: the vehicle must hold a setpoint or follow a reference while spending as little energy as possible. Because propulsion dominates the draw on a fixed onboard budget, the goal is to lower the average power the controller draws, not to cap the instantaneous power at any step: a lower average, sustained over a dive, stretches the same battery to a longer mission.

Classical model-based and proportional--integral--derivative (PID) controllers have long served underwater motion control, but they hinge on careful hand tuning and accurate hydrodynamic parameters that are notoriously hard to identify for the nonlinear, strongly coupled dynamics of underwater vehicles~\cite{fossen2011handbook}. Reinforcement learning (RL) offers a model-free alternative that learns a control policy directly from interaction and can track references well without an explicit model~\cite{carlucho2018adaptive}. An RL policy, however, optimizes exactly what its reward specifies: crediting task accuracy alone yields high-gain, oscillatory, actuation that reaches the target quickly, but wastes energy and stresses the thrusters, the very behavior an endurance-limited vehicle cannot afford.

The standard way to make a policy energy-aware adds an energy penalty to the reward~\cite{wen2024adaptive}, fixing the trade-off between task accuracy and average power through a single dimensionless weight. This weight has three drawbacks. It carries no physical units, so the resulting average power cannot be declared in advance and is known only after training. Its meaning depends on each vehicle's actuation and dynamics, so it must be re-specified by hand for every vehicle and task. And the search is fragile: a badly tuned weight can even increase average power relative to a task-only baseline, as our experiments confirm (see Section~\ref{sec:exp}). This raises the central question we address: \emph{Can the energy objective be modified from a hand-tuned, dimensionless reward weight into a control specification that is declared in watts, met automatically by a constraint, and applied by one procedure unchanged across vehicles and tasks, all without sacrificing task performance?}

We answer the question by treating average power as a budget rather than a reward term. We cast energy-efficient control as a constrained Markov decision process (CMDP)~\cite{altman1999constrained} that maximizes the task reward, subject to an explicit limit on the average thruster power (the time-average over a dive, not the instantaneous power at any step), and solves it with a proximal policy optimization (PPO)-Lagrangian algorithm~\cite{schulman2017proximal,ray2019benchmarking}. A single dual variable, updated online by dual ascent~\cite{stooke2020responsive}, takes the place of the hand-set weight: an average-power budget in watts is declared, a direct physical knob on the average-power operating point, and the dual variable adapts automatically to drive the policy's average power toward it, separately for each vehicle and task. The energy objective thus becomes a specification in physical units, pursued without any case-specific weight search. Fig.~\ref{fig:framework} gives an overview of the resulting framework, from the vehicle plant and thruster power model to the constrained formulation, the PPO-Lagrangian solver, and the multi-vehicle, multi-task validation.

\begin{figure*}[t]
\centering
\includegraphics[width=0.85\textwidth]{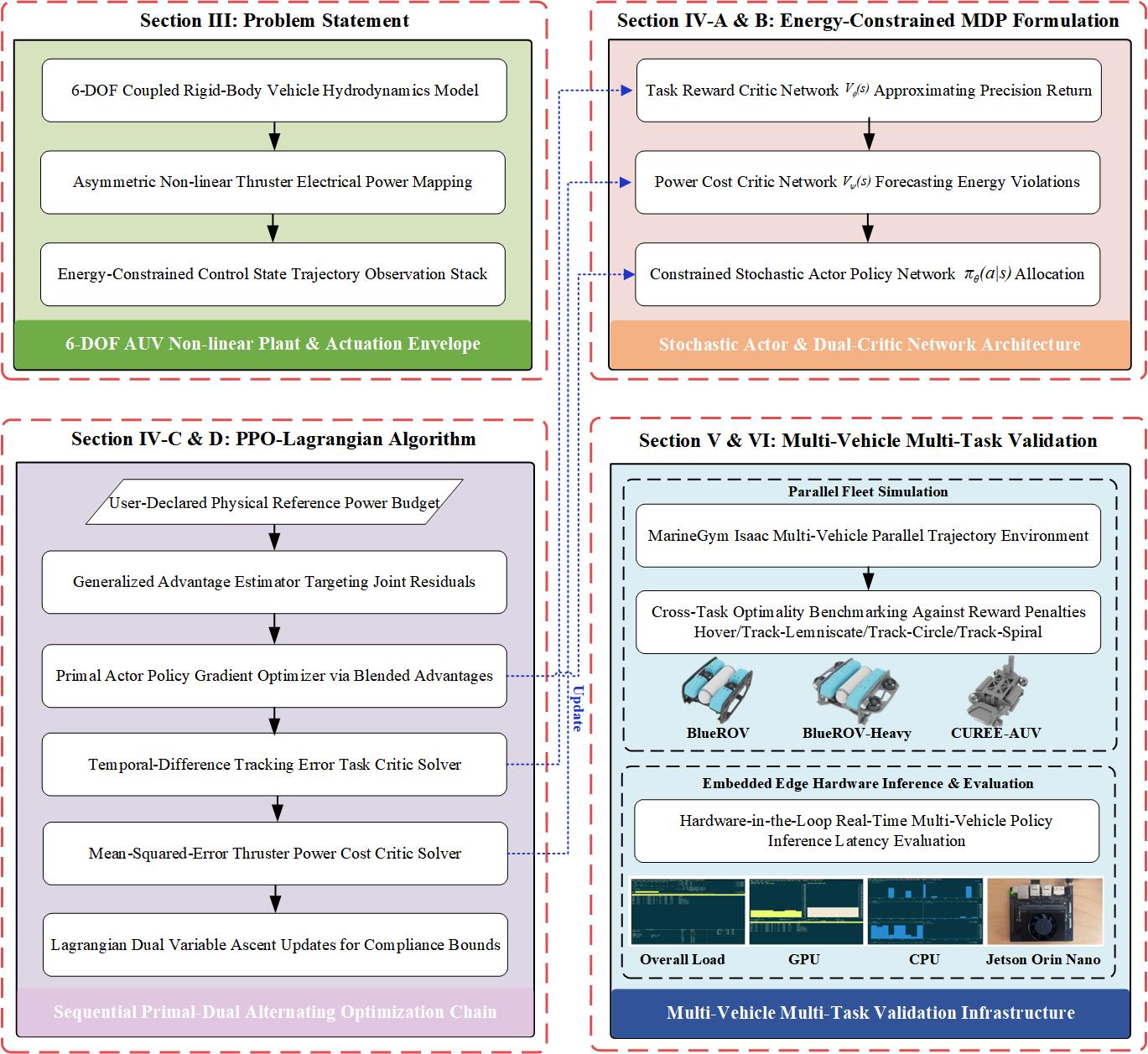}
\caption{Overview of the proposed average-power-budgeted underwater control framework.}
\label{fig:framework}
\end{figure*}

In summary, this paper makes the following contributions:
\begin{itemize}
\item \textbf{Reframing energy efficiency as a declarable budget.} We treat thruster average-power not as one more reward term but as an explicit budget in watts, posing energy-efficient underwater control as a constrained problem. To our knowledge, this is the first time underwater vehicle energy efficiency is cast as an explicit average-power-budget constraint.
\item \textbf{An average-power-constrained formulation solved by PPO-Lagrangian.} We bound the modeled average thruster power by a budget in watts and solve the resulting CMDP with a PPO-Lagrangian method. Unlike a reward penalty, the target average power is declared in advance in physical units, and an online dual variable replaces the hand-tuned weight, removing the per-vehicle, per-task weight search.
\item \textbf{A reproducible training framework with verified on-board efficiency.} We implement the full pipeline in the graphics processing unit (GPU)-accelerated MarineGym simulator~\cite{chu2025marinegym} and apply the same average-power-budgeting procedure across all vehicles and tasks; profiling on an embedded NVIDIA Jetson platform confirms the learned controller runs in real time, the constraint acting only during training and adding no inference-time cost.
\item \textbf{Controlled multi-vehicle, multi-task evidence.} Across three vehicles and four tasks, varying only the energy treatment, the constrained policy uses the least average power in every setting (14--32\% below the task-only baseline) and is the smoothest in most, while preserving task success and keeping tracking accuracy comparable to the baselines in most; the reward penalty, by contrast, is inconsistent and sometimes raises average power.
\end{itemize}

The remainder of this paper is organized as follows. Section~\ref{sec:related} reviews RL for underwater control, energy-aware control, and constrained RL. Section~\ref{sec:problem} states the physical control problem. Section~\ref{sec:method} formulates the CMDP and the PPO-Lagrangian algorithm. Section~\ref{sec:setup} describes the vehicles, tasks, simulator, and evaluation metrics. Section~\ref{sec:exp} presents and analyzes the results, while Section~\ref{sec:conclusion} concludes.

\section{Related Work}
\label{sec:related}

\subsection{Rule-Based and Model-Based Control for Underwater Vehicles}
Classical underwater motion control is predominantly model-based and hand-tuned. Proportional--integral--derivative loops, multivariable and adaptive sliding-mode autopilots~\cite{healey1993multivariable,cristi1990adaptive}, command-filtered backstepping~\cite{wang2018threedim}, Lyapunov-based model predictive control~\cite{shen2018trajectory}, and optimal-control trajectory planning~\cite{spangelo1994trajectory} have all been demonstrated for diving, steering, path following, and trajectory tracking, with observer-based designs added to reject unmeasured states and disturbances~\cite{peng2018output}; more recent work in this journal adds saturation-based adaptive tracking~\cite{tijjani2026saturation}, prescribed-performance and event-triggered control~\cite{shi2024event,batmani2021event,li2023discrete}, finite-time fault-tolerant control under thruster faults and input saturation~\cite{ali2025event}, and cooperative and formation control of multiple vehicles~\cite{wang2022cooperative,yao2025distributed}, surveyed alongside line-of-sight guidance for marine vehicles~\cite{gu2023advances,er2024intelligent}. Their common dependence, however, is on an accurate dynamic model: the added-mass, Coriolis, and nonlinear damping terms of the underwater equations of motion are strongly coupled and notoriously difficult to identify~\cite{fossen2011handbook}, so these controllers require careful per-vehicle modeling and gain tuning, and degrade when the hydrodynamics drift. This difficulty is the standing motivation for model-free methods that learn a controller directly from interaction.

\subsection{Reinforcement Learning for Underwater Vehicle Control}
Reinforcement learning sidesteps explicit modeling by optimizing a control policy directly from interaction, and has become a competitive alternative for underwater control~\cite{wang2025editorial}. Actor--critic and policy-gradient methods learn to stabilize and track references under disturbances without an identified model, including adaptive neural-network controllers that absorb input nonlinearities~\cite{cui2017adaptive}, deterministic policy gradients for depth and trajectory tracking~\cite{wu2019depth,shi2019multipseudo}, learning-based station keeping~\cite{knudsen2019deep}, and deep policies for path following, target tracking, and output-constrained control~\cite{carlucho2018adaptive,zhang2020interactive,wang2022target,elhaki2021saturated,ma2024neural}, with proximal-policy-optimization path planning and tracking~\cite{he2022asynchronous} and twin-delayed deep deterministic policy gradients for continuous-action path planning~\cite{chu2024local} among recent examples in this journal. GPU-accelerated simulators such as MarineGym~\cite{chu2025marinegym}, built on Isaac Gym~\cite{makoviychuk2021isaac}, make large-scale multi-vehicle training practical. A policy, however, optimizes only what its reward encodes, and the bulk of this work rewards task accuracy alone, leaving thruster energy uncosted. The few energy-aware variants append an energy or action-effort penalty to the reward~\cite{wen2024adaptive}, an idea also pursued for energy-aware robotic manipulation in infrastructure operation~\cite{tao2026energy}, but this scalarization carries no physical units: the resulting average power cannot be declared in advance, the weight value does not transfer across vehicles or tasks, and a badly tuned weight can even raise the average power, as we quantify in Section~\ref{sec:exp}. An average-power budget is a resource limit, and a penalty term cannot express a limit, only a balance, which points to treating average power as an explicit constraint.

\subsection{Constrained Reinforcement Learning}
Constrained Markov decision processes~\cite{altman1999constrained} maximize reward subject to explicit bounds on auxiliary cumulative costs, the natural language for a budget, and recent theory shows the constrained-RL problem has zero duality gap, justifying primal--dual solutions~\cite{paternain2019zerogap}. A mature toolkit follows: the constrained policy optimization (CPO) algorithm enforces the bound within a trust region~\cite{achiam2017constrained}; Lagrangian and interior-point methods convert the bound into an adaptive penalty whose multiplier is learned online~\cite{tessler2019reward,liu2020ipo}, with PID-style multiplier updates improving stability~\cite{stooke2020responsive}; and Lyapunov-based schemes guarantee feasibility throughout training~\cite{chow2018lyapunov}. Yet, this machinery is almost exclusively developed and evaluated for \emph{safety}, that is, collision and state-constraint satisfaction on locomotion and manipulation benchmarks~\cite{ray2019benchmarking,garcia2015comprehensive,yang2023barrier}, while the rare uses of constrained RL as an \emph{energy} budget target ground systems, such as Lagrangian energy management for hybrid electric vehicles~\cite{zhang2021lagrangian}, not underwater platforms. 

To our knowledge, no prior work casts average thruster power as an explicit, declarable budget enforced by a constraint for underwater vehicles, nor evaluates such a formulation across multiple vehicles and tasks; this is the gap we fill. Table~\ref{tab:related_work} situates our work against representative approaches along two axes: modeling and scope, namely whether a method is model-free, validated underwater, and evaluated across multiple vehicles; and how energy is specified, namely whether it is targeted at all (energy-aware), formulated as a constraint rather than a reward penalty (constraint-based), governed by an adaptive multiplier rather than a fixed weight, or declarable a priori in watts.

\begin{table*}[tb]
\centering
\caption{Positioning the proposed controller among representative approaches to energy in learning-based and model-based motion control. $\checkmark$, $\sim$, and $\times$ denote explicit, partial, and no support; the columns are defined in the text.}
\label{tab:related_work}
\small
\renewcommand{\arraystretch}{1.2}
\setlength{\tabcolsep}{3pt}
\begin{tabular}{@{}p{0.32\linewidth} *{7}{c}@{}}
\toprule
\textbf{Approach} & Model-free & Underwater & Multi-veh. & Energy-aware & Constraint & Adaptive & Declarable~(W) \\
\midrule
Deep RL, task reward~\cite{carlucho2018adaptive,wu2019depth,cui2017adaptive,ma2024neural} & $\checkmark$ & $\checkmark$ & $\times$ & $\times$ & $\times$ & $\times$ & $\times$ \\
Deep RL, energy-penalty reward~\cite{knudsen2019deep,wen2024adaptive} & $\checkmark$ & $\checkmark$ & $\times$ & $\checkmark$ & $\times$ & $\times$ & $\times$ \\
Model-based / optimal control~\cite{fossen2011handbook,spangelo1994trajectory,shen2018trajectory} & $\times$ & $\checkmark$ & $\times$ & $\checkmark$ & $\sim$ & $\times$ & $\checkmark$ \\
Constrained RL, trust region (CPO)~\cite{achiam2017constrained,liu2020ipo} & $\checkmark$ & $\times$ & $\sim$ & $\times$ & $\checkmark$ & $\checkmark$ & $\sim$ \\
Constrained RL, Lagrangian~\cite{altman1999constrained,ray2019benchmarking,stooke2020responsive,tessler2019reward,yang2023barrier,zhang2021lagrangian} & $\checkmark$ & $\times$ & $\sim$ & $\times$ & $\checkmark$ & $\checkmark$ & $\sim$ \\
\textbf{Ours} & $\checkmark$ & $\checkmark$ & $\checkmark$ & $\checkmark$ & $\checkmark$ & $\checkmark$ & $\checkmark$ \\
\bottomrule
\end{tabular}
\end{table*}

\section{Problem Statement}
\label{sec:problem}

\subsection{Vehicle Dynamics}
In the body-fixed frame, the six-degree-of-freedom rigid-body, and hydrodynamic, motion of an underwater vehicle obeys the Fossen model~\cite{fossen2011handbook}
\begin{align}
\dot{\boldsymbol{\eta}} &= \boldsymbol{J}(\boldsymbol{\eta})\,\boldsymbol{\nu}, \label{eq:kinematics}\\
\boldsymbol{M}\dot{\boldsymbol{\nu}} + \boldsymbol{C}(\boldsymbol{\nu})\boldsymbol{\nu} + \boldsymbol{D}(\boldsymbol{\nu}_r)\,\boldsymbol{\nu}_r + \boldsymbol{g}(\boldsymbol{\eta}) &= \boldsymbol{\tau} = \boldsymbol{B}\,\boldsymbol{u}, \label{eq:dynamics}
\end{align}
where $\boldsymbol{\eta}=[x,y,z,\phi,\theta,\psi]^\top$ is the inertial pose, $\boldsymbol{\nu}=[\boldsymbol{\nu}_1^\top,\boldsymbol{\nu}_2^\top]^\top=[u,v,w,p,q,r]^\top$ the body twist with linear part $\boldsymbol{\nu}_1$ and angular part $\boldsymbol{\nu}_2$, $\boldsymbol{\nu}_r=\boldsymbol{\nu}-\boldsymbol{\nu}_c$ its value relative to the ocean current $\boldsymbol{\nu}_c$, and $\boldsymbol{u}=[f_1,\dots,f_M]^\top$ the $M$ thruster forces mapped by $\boldsymbol{B}$ to the body-frame generalized force $\boldsymbol{\tau}$, that is, the stacked force and torque acting on the vehicle. The kinematic transform is block-diagonal in the body-to-inertial rotation $\boldsymbol{R}(\boldsymbol{\Theta})$ and the Euler-rate map $\boldsymbol{T}(\boldsymbol{\Theta})$, with $\boldsymbol{\Theta}=[\phi,\theta,\psi]^\top$,
\begin{equation}
\boldsymbol{J}(\boldsymbol{\eta})=\begin{bmatrix}\boldsymbol{R}(\boldsymbol{\Theta}) & \boldsymbol{0}_3\\[2pt] \boldsymbol{0}_3 & \boldsymbol{T}(\boldsymbol{\Theta})\end{bmatrix}.
\label{eq:transform}
\end{equation}
The generalized inertia and Coriolis matrices split into rigid-body and added-mass parts, $\boldsymbol{M}=\boldsymbol{M}_{\mathrm{RB}}+\boldsymbol{M}_A$ and $\boldsymbol{C}=\boldsymbol{C}_{\mathrm{RB}}+\boldsymbol{C}_A$, with
\begin{equation}
\begin{aligned}
\boldsymbol{M}_{\mathrm{RB}}&=\begin{bmatrix} m\boldsymbol{I}_3 & -m\boldsymbol{S}(\boldsymbol{r}_g)\\[2pt] m\boldsymbol{S}(\boldsymbol{r}_g) & \boldsymbol{I}_g \end{bmatrix},\\[2pt]
\boldsymbol{M}_A&=-\operatorname{diag}(X_{\dot u},Y_{\dot v},Z_{\dot w},K_{\dot p},M_{\dot q},N_{\dot r}),
\end{aligned}
\label{eq:massmats}
\end{equation}
\begin{equation}
\boldsymbol{C}_A(\boldsymbol{\nu})=-\begin{bmatrix} \boldsymbol{0}_3 & \boldsymbol{S}(\boldsymbol{M}_{A,t}\boldsymbol{\nu}_1)\\[2pt] \boldsymbol{S}(\boldsymbol{M}_{A,t}\boldsymbol{\nu}_1) & \boldsymbol{S}(\boldsymbol{M}_{A,r}\boldsymbol{\nu}_2) \end{bmatrix},
\label{eq:coriolis}
\end{equation}
where $\boldsymbol{S}(\cdot)$ is the skew-symmetric cross-product operator, $m,\boldsymbol{I}_g,\boldsymbol{r}_g$ the mass, inertia tensor, and center of gravity, respectively, $X_{\dot u},\dots,N_{\dot r}$ the added-mass coefficients, and $\boldsymbol{M}_{A,t},\boldsymbol{M}_{A,r}$ the translational and rotational blocks of $\boldsymbol{M}_A$, respectively. Hydrodynamic drag superposes linear and quadratic damping on the relative velocity,
\begin{equation}
\boldsymbol{D}(\boldsymbol{\nu}_r)=\boldsymbol{D}_{\mathrm{l}}+\boldsymbol{D}_{\mathrm{q}}\,\lvert\boldsymbol{\nu}_r\rvert,
\label{eq:damping}
\end{equation}
and gravity and buoyancy give the restoring generalized force
\begin{equation}
\boldsymbol{g}(\boldsymbol{\eta})=\begin{bmatrix}
(W\!-\!B)\sin\theta\\
-(W\!-\!B)\cos\theta\sin\phi\\
-(W\!-\!B)\cos\theta\cos\phi\\
-z_bB\cos\theta\sin\phi\\
-z_bB\sin\theta\\
0
\end{bmatrix},
\label{eq:restoring}
\end{equation}
with weight $W$, buoyancy $B$, and center-of-buoyancy offset $z_b$. We take this standard model as the given plant; its strongly coupled, hard-to-identify, hydrodynamics are what motivate the model-free controller of Section~\ref{sec:method}.

\subsection{Thruster Power Model}
Energy efficiency is defined on the electrical power the thrusters actually draw, not on a proxy. All three vehicles are driven by Blue Robotics T200 thrusters~\cite{bluerobotics_t200}, so a single power map applies throughout. The policy does not command thrust directly: each thruster receives a normalized command in $[-1,1]$, which the T200 datasheet maps successively to a rotational speed, a thrust force $f$, and the electrical power drawn. Composing these datasheet maps expresses the per-thruster power as an asymmetric power law in the produced thrust,
\begin{equation}
P(f) =
\begin{cases}
a_{\mathrm{f}}\,f^{\,b_{\mathrm{f}}}, & f \ge 0,\\[2pt]
a_{\mathrm{r}}\,|f|^{\,b_{\mathrm{r}}}, & f < 0,
\end{cases}
\label{eq:t200}
\end{equation}
with forward coefficients $(a_{\mathrm{f}},b_{\mathrm{f}})$ and reverse coefficients $(a_{\mathrm{r}},b_{\mathrm{r}})$ obtained from the datasheet; both exponents exceed one, so power grows super-linearly with thrust and, for equal magnitude, is higher in reverse. Each per-thruster thrust $f_{i,t}$ is produced by the policy's throttle command through these datasheet maps, so the instantaneous power at step $t$ is the sum over the $M$ thrusters, $c_t = \sum_{i=1}^{M} P(f_{i,t})$ (in watts), and its time-average over the task duration $T$ (the episode horizon) is the average power
\begin{equation}
\bar{P} = \frac{1}{T}\sum_{t=0}^{T-1} c_t .
\label{eq:power}
\end{equation}
Here $c_t$ is a modeled electrical power in watts, not a measured current, so a limit placed on it carries a physical unit. We keep three quantities distinct throughout: the instantaneous power $c_t$, its task-average the average power $\bar P$, and the onboard energy budget, the battery capacity in joules that ultimately limits endurance. For a fixed task duration $T$, the energy drawn over a task is $E=\bar P\,T$, so bounding $\bar P$ caps the per-task energy and, on a fixed battery, extends endurance in proportion. We therefore constrain the average power $\bar P$, which is declarable in watts and shared across tasks of different length, rather than a task-dependent energy in joules; this is what later allows the average-power budget to be declared in advance, a property no dimensionless penalty weight possesses.

\subsection{Optimization Objective}
The control problem is two-sided. Under the dynamics~\eqref{eq:kinematics}--\eqref{eq:dynamics}, the vehicle must perform its task, namely station-keeping at a setpoint or tracking a reference trajectory, while holding its average thruster power within a declared budget,
\begin{equation}
\underset{\text{controller}}{\text{minimize}}\ \ \text{task error}\qquad\text{subject to}\qquad \bar{P}\le d ,
\label{eq:problem}
\end{equation}
where the task error is the distance to the setpoint or reference, $\bar P$ is the average power of~\eqref{eq:power}, and $d$ is the budget in watts. Equation~\eqref{eq:problem} states only what is required: the task, achieved under an explicit and physically meaningful average-power limit. How to realize it, in particular the choice to keep the task objective intact and impose the average-power limit as a constraint solved by reinforcement learning, is the subject of Section~\ref{sec:method}.

\section{Average-Power-Constrained Policy Optimization}
\label{sec:method}

\subsection{Underwater Control as an MDP}
We model the task of Section~\ref{sec:problem} as a discrete-time Markov decision process (MDP) $(\mathcal{S},\mathcal{A},P,r,\gamma)$, obtained by sampling~\eqref{eq:kinematics}--\eqref{eq:dynamics} at the control rate, with transition kernel $P(\boldsymbol{s}_{t+1}\mid\boldsymbol{s}_t,\boldsymbol{a}_t)$ induced by the integrated dynamics and the domain randomization of Section~\ref{sec:setup}. The state $\boldsymbol{s}_t\in\mathcal{S}$ stacks the vehicle pose and twist $(\boldsymbol{\eta},\boldsymbol{\nu})$, the body-axis orientation, and the previous thruster command with the task reference (the hover setpoint or the upcoming waypoints); the action $\boldsymbol{a}_t\in\mathcal{A}$ is the vector of normalized thruster commands in $[-1,1]$ mapped to the forces $\boldsymbol{u}$ through $\boldsymbol{B}$; and the policy $\pi_\theta(\boldsymbol{a}_t\mid\boldsymbol{s}_t)$ is a Gaussian with neural-network mean and state-independent log-standard-deviation. The task reward $r_t = r(\boldsymbol{s}_t,\boldsymbol{a}_t)$ rewards proximity to the setpoint or reference, with auxiliary shaping for uprightness and low body spin. This reward is \emph{identical} across the three methods we compare, so that any difference in behavior is attributable solely to how energy is treated, the premise of the controlled comparison in Section~\ref{sec:exp}.

For constraint analysis, we use the discounted state--action occupancy measure induced by $\pi$,
\begin{equation}
\rho^\pi(\boldsymbol{s},\boldsymbol{a}) = (1-\gamma)\sum_{t\ge0}\gamma^{t}\,\Pr{}_{\!\pi}\!\bigl(\boldsymbol{s}_t=\boldsymbol{s},\,\boldsymbol{a}_t=\boldsymbol{a}\bigr),
\label{eq:occupancy}
\end{equation}
and, for either per-step signal $z\in\{r,c\}$ (task reward $r_t$ or power cost $c_t$ of~\eqref{eq:t200}), its return, state value, action value, and advantage
\begin{equation}
\begin{aligned}
J_z(\pi) &= \tfrac{1}{1-\gamma}\,\mathbb{E}_{(\boldsymbol{s},\boldsymbol{a})\sim\rho^\pi}\!\bigl[z(\boldsymbol{s},\boldsymbol{a})\bigr],\\
V^z_\pi(\boldsymbol{s}) &= \mathbb{E}_\pi\!\Bigl[\textstyle\sum_{k\ge0}\gamma^{k} z_{t+k}\,\big|\,\boldsymbol{s}_t=\boldsymbol{s}\Bigr],\\
Q^z_\pi(\boldsymbol{s},\boldsymbol{a}) &= z(\boldsymbol{s},\boldsymbol{a}) + \gamma\,\mathbb{E}_{\boldsymbol{s}'\sim P}\!\bigl[V^z_\pi(\boldsymbol{s}')\bigr],\\
A^z_\pi(\boldsymbol{s},\boldsymbol{a}) &= Q^z_\pi(\boldsymbol{s},\boldsymbol{a}) - V^z_\pi(\boldsymbol{s}),
\end{aligned}
\label{eq:valuefns}
\end{equation}
so that $z=r$ gives the discounted task return $J_r$, the maximization objective. The constrained quantity is the average thruster power $J_c(\pi)\triangleq\mathbb{E}_{(\boldsymbol{s},\boldsymbol{a})\sim\rho^\pi}[c]$, in watts, the occupancy-weighted mean of the per-step power $c_t$; being linear in $\rho^\pi$, it makes the constrained problem convex in $\rho^\pi$, a fact we use in Section~\ref{sec:theory}. During training $J_c$ is estimated by the empirical episode-average power $\hat{J}_c=\tfrac{1}{T}\sum_{t=0}^{T-1}c_t=\bar P$ of~\eqref{eq:power}, which the dual update drives toward $d$ and which Section~\ref{sec:setup} reports; the occupancy-weighted mean $J_c$ and this uniform episode average coincide as $\gamma\to1$ (here $\gamma=0.99$).

\subsection{Average Power as an Explicit Constraint}
We keep the task reward intact and impose the average-power budget of~\eqref{eq:problem} as an explicit constraint, yielding the constrained MDP~\cite{altman1999constrained}
\begin{equation}
\max_{\pi}\; J_r(\pi)\quad\text{s.t.}\quad J_c(\pi)\le d ,
\label{eq:cmdp}
\end{equation}
where $J_c(\pi)$ is the average thruster power defined above and $d$ the watt budget of~\eqref{eq:problem}. Casting the constraint this way gives the objective two properties at once. First, it is \emph{declarable}: the average operating power is fixed in advance, in physical units, rather than emerging post hoc from training. Second, the constraint acts on the \emph{correct operand}, the modeled power $c_t$ of~\eqref{eq:t200}, rather than on a surrogate of action magnitude. The resulting controller, PPO-Lag, solves~\eqref{eq:cmdp}; Section~\ref{sec:exp} evaluates it against task-only and reward-penalty baselines defined there.

\subsection{PPO-Lagrangian}
\emph{Cost critic.} The constraint in~\eqref{eq:cmdp} is handled online by a dedicated cost critic $V^c_\psi$ that bootstraps the budget-centered power $c_t-d$; its Bellman operator is a $\gamma$-contraction, so the critic is well posed (Lemma~\ref{lem:contraction}). In practice $V^c_\psi$ is fit by regression to the cost returns, and the cost advantage is formed by generalized advantage estimation (GAE)~\cite{schulman2016high} on the budget-centered temporal-difference residual,
\begin{equation}
\delta^c_t = (c_t-d) + \gamma V^c_\psi(\boldsymbol{s}_{t+1}) - V^c_\psi(\boldsymbol{s}_t),
\qquad
A^c_t = \sum_{l\ge0}(\gamma\kappa)^{l}\,\delta^c_{t+l},
\label{eq:gae}
\end{equation}
with GAE parameter $\kappa$; the reward advantage $A^r_t$ is computed identically from $\delta^r_t = r_t + \gamma V^r_\phi(\boldsymbol{s}_{t+1}) - V^r_\phi(\boldsymbol{s}_t)$. Because $\delta^c_t$ is centered at the budget, $A^c_t>0$ signals that the policy is locally over budget.

\emph{Lagrangian and saddle point.} We solve~\eqref{eq:cmdp} through its Lagrangian~\cite{altman1999constrained,ray2019benchmarking}
\begin{equation}
\mathcal{L}(\pi,\lambda) = J_r(\pi) - \lambda\bigl(J_c(\pi)-d\bigr),\qquad \lambda\ge 0 ,
\label{eq:lagrangian}
\end{equation}
and seek the saddle point $\max_{\pi}\min_{\lambda\ge0}\mathcal{L}(\pi,\lambda)$ by alternating a primal policy update at fixed $\lambda$ with a dual update of $\lambda$.

\emph{Primal update.} For fixed $\lambda$, $\mathcal{L}(\cdot,\lambda)$ is an unconstrained RL objective with the composite per-step signal $r_t-\lambda(c_t-d)$. Applying the policy-gradient theorem to the value functions~\eqref{eq:valuefns} and using the linearity of the advantage in this signal,
\begin{equation}
\nabla_\theta\mathcal{L}(\pi_\theta,\lambda)
= \mathbb{E}_{(\boldsymbol{s},\boldsymbol{a})\sim\rho^{\pi_\theta}}\!\Bigl[\nabla_\theta\log\pi_\theta(\boldsymbol{a}\!\mid\!\boldsymbol{s})\,\bigl(A^r_t - \lambda A^c_t\bigr)\Bigr],
\label{eq:pg}
\end{equation}
so the reward and cost advantages enter through a single Lagrangian advantage
\begin{equation}
A_t(\lambda) = A^r_t - \lambda\, A^c_t ,
\label{eq:surrogate}
\end{equation}
with $A^r_t,A^c_t$ the estimates of~\eqref{eq:gae}. We realize the ascent in~\eqref{eq:pg} by the clipped PPO surrogate~\cite{schulman2017proximal}
\begin{equation}
\begin{aligned}
L^{\mathrm{CLIP}}(\theta) = \mathbb{E}_t\Bigl[\min\bigl(&\varrho_t(\theta)\,A_t(\lambda),\\
&\mathrm{clip}(\varrho_t(\theta),1\!-\!\epsilon,1\!+\!\epsilon)\,A_t(\lambda)\bigr)\Bigr],
\end{aligned}
\label{eq:clip}
\end{equation}
where $\varrho_t(\theta)=\pi_\theta(\boldsymbol{a}_t\!\mid\!\boldsymbol{s}_t)/\pi_{\theta_{\mathrm{old}}}(\boldsymbol{a}_t\!\mid\!\boldsymbol{s}_t)$ is the importance ratio and $\epsilon$ the clip range, while $V^r_\phi,V^c_\psi$ are fit by regression to their respective returns.

\emph{Dual update.} The multiplier performs subgradient ascent on the dual function $g(\lambda)=\max_\pi\mathcal{L}(\pi,\lambda)$, whose subgradient is exactly the constraint violation (Proposition~\ref{prop:duality}),
\begin{equation}
\lambda \leftarrow \mathrm{clip}\!\bigl(\lambda + \eta_\lambda\,(\hat{J}_c-d),\; \lambda_{\min},\, \lambda_{\max}\bigr),
\label{eq:dual}
\end{equation}
with $\hat{J}_c$ the batch-mean average power and $\eta_\lambda$ the dual step size. To keep $\lambda>0$ without a barrier, our implementation parameterizes the multiplier in log-space, $\nu=\log\lambda$, and updates $\nu$ by minimizing the surrogate $\mathcal{L}_\lambda=-\nu\,(\hat{J}_c-d)$ (whose gradient is $-(\hat{J}_c-d)$), which is the ascent step
\begin{equation}
\begin{aligned}
\nu &\leftarrow \nu + \eta_\lambda\,(\hat{J}_c-d),\\
\lambda &= \exp\!\bigl(\mathrm{clip}(\nu,\,\log\lambda_{\min},\,\log\lambda_{\max})\bigr),
\end{aligned}
\label{eq:logspace}
\end{equation}
a log-space variant of~\eqref{eq:dual} that raises $\lambda$ when the batch is over budget ($\hat{J}_c>d$) and lowers it otherwise, with $\nu$ clipped so that $\lambda\in[\lambda_{\min},\lambda_{\max}]$ throughout~\cite{stooke2020responsive}; in practice $\nu$ is carried by the Adam optimizer~\cite{kingma2015adam}.

\emph{Networks and overall objective.} The actor $\pi_\theta$ and the two critics $V^r_\phi,V^c_\psi$ are multilayer perceptrons (MLPs) trained jointly. Over each rollout, we minimize
\begin{equation}
\begin{aligned}
\mathcal{J}(\theta,\phi,\psi) ={}& -L^{\mathrm{CLIP}}(\theta) + c_v\bigl(\mathcal{L}^{V}_r(\phi)+\mathcal{L}^{V}_c(\psi)\bigr)\\
&- c_e\,\mathbb{E}_t\!\bigl[\mathcal{H}\bigl(\pi_\theta(\cdot\mid\boldsymbol{s}_t)\bigr)\bigr],
\end{aligned}
\label{eq:totalloss}
\end{equation}
where $\mathcal{H}$ is the policy entropy with coefficient $c_e$ (sustaining exploration), $c_v$ weights the value losses, and each critic uses the PPO clipped-regression loss toward its return target $\hat{G}_t$,
\begin{equation}
\mathcal{L}^{V} = \mathbb{E}_t\!\Bigl[\max\!\bigl((V-\hat{G}_t)^2,\ (\mathrm{clip}(V,V_{\mathrm{old}}\!-\!\epsilon,V_{\mathrm{old}}\!+\!\epsilon)-\hat{G}_t)^2\bigr)\Bigr],
\label{eq:valueloss}
\end{equation}
evaluated with the reward and cost returns $\hat{G}^r_t,\hat{G}^c_t$ for $V^r_\phi$ and $V^c_\psi$. The actor and critics are optimized by Adam over several epochs of minibatches per rollout, and the multiplier $\nu$ by its own dual step~\eqref{eq:logspace}.

What makes this effective is that $\lambda$ is not a hyperparameter but a \emph{state of the optimizer}: it multiplies the budget-centered signal $c_t-d$, rises when the batch exceeds the watt budget $d$ and falls otherwise, and converges to the value that drives the average power toward the budget, separately for each vehicle and task. The only fixed quantities are the dual step $\eta_\lambda$ and the bounds $[\lambda_{\min},\lambda_{\max}]$, held constant across all settings; what is eliminated is the per-setting search for the trade-off itself, which $\lambda$ discovers online.

\subsection{Theoretical Properties}
\label{sec:theory}
We first note that the cost critic is well posed (Lemma~\ref{lem:contraction}), then analyze the saddle point~\eqref{eq:lagrangian} in two regimes. When the budget is feasible (Assumption~\ref{ass:reg}), the dual update drives the policy to it (Lemma~\ref{lem:bounded}, Propositions~\ref{prop:duality}--\ref{prop:violation}, Remark~\ref{rem:shadow}); when $d$ is set below what the task admits, the multiplier saturates and Remark~\ref{rem:saturation} governs the outcome. The statements are standard for constrained MDPs~\cite{altman1999constrained,paternain2019zerogap} and concern the idealized exact-gradient dynamics, qualified under clipped PPO updates and function approximation by Remark~\ref{rem:scope}.

We begin with the cost critic, which targets the fixed point of the cost Bellman operator $\mathcal{T}^\pi_c$ acting on any bounded $V:\mathcal{S}\to\mathbb{R}$ by
\begin{equation}
(\mathcal{T}^\pi_c V)(\boldsymbol{s}) = \mathbb{E}_{\boldsymbol{a}\sim\pi,\,\boldsymbol{s}'\sim P}\!\bigl[(c(\boldsymbol{s},\boldsymbol{a})-d) + \gamma\,V(\boldsymbol{s}')\bigr].
\label{eq:costbellman}
\end{equation}
\begin{lemma}[$\gamma$-contraction of the cost Bellman operator]
\label{lem:contraction}
For all bounded $V_1,V_2$, $\;\lVert\mathcal{T}^\pi_c V_1-\mathcal{T}^\pi_c V_2\rVert_\infty\le\gamma\,\lVert V_1-V_2\rVert_\infty$; hence $\mathcal{T}^\pi_c$ has a unique fixed point $V^c_\pi$, and value iteration $V_{k+1}=\mathcal{T}^\pi_c V_k$ converges to it geometrically.
\end{lemma}
\begin{proof}
Subtracting two applications of~\eqref{eq:costbellman}, the action cost and budget terms cancel, leaving, for every $\boldsymbol{s}$,
\[
\begin{aligned}
\bigl|(\mathcal{T}^\pi_c V_1-\mathcal{T}^\pi_c V_2)(\boldsymbol{s})\bigr|
&= \gamma\,\bigl|\mathbb{E}_{\boldsymbol{s}'}[V_1(\boldsymbol{s}')-V_2(\boldsymbol{s}')]\bigr|\\
&\le \gamma\,\lVert V_1-V_2\rVert_\infty .
\end{aligned}
\]
Taking the supremum over $\boldsymbol{s}$ gives the contraction; since $\gamma<1$, the Banach fixed-point theorem~\cite{puterman1994markov} yields uniqueness and geometric convergence.
\end{proof}

\begin{assumption}[Regularity and feasibility]
\label{ass:reg}
The returns $J_r(\pi)$ and $J_c(\pi)$ are bounded and differentiable in the policy parameters; the dual step size $\eta_\lambda$ is sufficiently small; and there exists a strictly feasible policy, i.e.\ a policy $\pi$ with $J_c(\pi)<d$ (Slater's condition~\cite{boyd2004convex}).
\end{assumption}

\begin{lemma}[Bounded multiplier]
\label{lem:bounded}
For every iteration, the multiplier produced by~\eqref{eq:dual} satisfies $\lambda\in[\lambda_{\min},\lambda_{\max}]$.
\end{lemma}
\begin{proof}
Immediate, from the projection in~\eqref{eq:dual}, which clips the updated value to $[\lambda_{\min},\lambda_{\max}]$ at every step.
\end{proof}

\noindent By Lemma~\ref{lem:bounded} the Lagrangian advantage~\eqref{eq:surrogate} is a bounded combination of $A^r_t$ and $A^c_t$ throughout training, so the multiplier cannot diverge; this is a stability property the log-space, clipped update shares with a fixed weight $w_E$ but a raw, unprojected dual ascent does not.

\begin{proposition}[Convex dual and strong duality]
\label{prop:duality}
The dual function $g(\lambda)=\max_{\pi}\mathcal{L}(\pi,\lambda)$ is convex in $\lambda$, and the update~\eqref{eq:dual} is a projected subgradient step on the dual problem $\min_{\lambda\ge0} g(\lambda)$. Under Assumption~\ref{ass:reg}, the constrained problem~\eqref{eq:cmdp} has zero duality gap in its occupancy-measure form, so a saddle point of~\eqref{eq:lagrangian} attains the constrained optimum of~\eqref{eq:cmdp}.
\end{proposition}
\begin{proof}
For fixed $\pi$, $\mathcal{L}(\pi,\lambda)$ is affine in $\lambda$, so $g$ is a pointwise maximum of affine functions and hence convex; by Danskin's theorem~\cite{boyd2004convex} a subgradient of $g$ at $\lambda$ is $-(J_c(\pi^\star_\lambda)-d)$, which is the increment applied in~\eqref{eq:dual}. Zero duality gap, under Slater's condition, is the standard CMDP result~\cite{altman1999constrained,paternain2019zerogap}.
\end{proof}

\begin{proposition}[Asymptotic constraint satisfaction]
\label{prop:violation}
Under Assumption~\ref{ass:reg}, the idealized primal--dual iterates $\{(\pi_k,\lambda_k)\}$ satisfy, for some constant $C>0$,
\begin{equation}
\frac{1}{K}\sum_{k=0}^{K-1}\bigl(J_c(\pi_k)-d\bigr)_+ \;\le\; \frac{C}{\sqrt{K}},\qquad K\ge 1,
\label{eq:violation}
\end{equation}
where $(x)_+=\max(x,0)$ and $K$ indexes the dual iterations. The running-average budget overshoot therefore vanishes as $K\to\infty$, and every limit point of $\{\pi_k\}$ is feasible.
\end{proposition}
\begin{proof}
Treat the primal--dual iterations as stochastic gradient descent--ascent on $\mathcal{L}$, and take the Lyapunov function $\Phi_k=\tfrac{1}{2\eta_\lambda}(\lambda_k-\lambda^\star)^2$. Because projection onto $[\lambda_{\min},\lambda_{\max}]$ is non-expansive toward the feasible $\lambda^\star$, the dual step~\eqref{eq:dual} gives
\[
\Phi_{k+1}-\Phi_k \;\le\; (\lambda_k-\lambda^\star)\bigl(J_c(\pi_k)-d\bigr) + \tfrac{\eta_\lambda}{2}\bigl(J_c(\pi_k)-d\bigr)^2 .
\]
By Lemma~\ref{lem:bounded}, the iterates are bounded, so $(J_c(\pi_k)-d)^2\le G^2$ for some $G$; convexity of the dual (Proposition~\ref{prop:duality}) together with Slater's condition (Assumption~\ref{ass:reg}) yields a constant $\beta>0$ with $(\lambda_k-\lambda^\star)(J_c(\pi_k)-d)\ge \beta\,(J_c(\pi_k)-d)_+$. Substituting, summing from $k=0$ to $K-1$, telescoping with $\Phi_k\ge0$, and dividing by $K$,
\[
\frac{1}{K}\sum_{k=0}^{K-1}\bigl(J_c(\pi_k)-d\bigr)_+ \;\le\; \frac{\Phi_0}{\beta\,K} + \frac{\eta_\lambda G^2}{2\beta}.
\]
Choosing $\eta_\lambda=\Theta(1/\sqrt{K})$ makes the right-hand side $O(1/\sqrt{K})$, which is~\eqref{eq:violation}; hence the running-average overshoot vanishes and every limit point of $\{\pi_k\}$ is feasible. The constrained-RL instantiation is detailed in~\cite{ray2019benchmarking,stooke2020responsive}.
\end{proof}

\begin{remark}[KKT conditions and the multiplier as a shadow price]
\label{rem:shadow}
At a saddle point $(\pi^\star,\lambda^\star)$ of~\eqref{eq:lagrangian}, the Karush--Kuhn--Tucker (KKT) conditions hold:
\begin{equation}
\begin{aligned}
&\nabla_\pi\mathcal{L}(\pi^\star,\lambda^\star)=0,\qquad J_c(\pi^\star)\le d,\\
&\lambda^\star\ge0,\qquad \lambda^\star\bigl(J_c(\pi^\star)-d\bigr)=0.
\end{aligned}
\label{eq:kkt}
\end{equation}
Complementary slackness (the last equality in~\eqref{eq:kkt}) gives a dichotomy: either the budget is active, $J_c(\pi^\star)=d$ with $\lambda^\star\ge0$, or it is slack, $J_c(\pi^\star)<d$ with $\lambda^\star=0$. By the stationarity condition in~\eqref{eq:kkt} and the envelope theorem~\cite{boyd2004convex}, $\lambda^\star=-\,\partial J_r^\star/\partial d$, so $\lambda^\star$ is the shadow price of the average-power budget, the marginal task-return cost of one watt, discovered by~\eqref{eq:dual} rather than fixed in advance. The multiplier in~\eqref{eq:surrogate} thus plays the role a hand-set penalty weight would, but is selected automatically for the declared budget $d$.
\end{remark}

\begin{remark}[Scope of the analysis]
\label{rem:scope}
The properties above concern the idealized exact-gradient dynamics. With clipped PPO updates, a non-convex policy class, and finite-sample estimates of $J_c$, the inner maximization is solved only approximately, so constraint satisfaction holds in the running-average, in-expectation sense of~\eqref{eq:violation} rather than per episode, and convergence is to a first-order stationary point rather than a global optimum. A further approximation is that the cost critic $V^c_\psi$ bootstraps the budget-centered power with discount $\gamma$ and thus serves only to estimate the cost advantage~\eqref{eq:gae}, whereas the constraint is enforced on the undiscounted episode-average power $\hat{J}_c$ through the dual update~\eqref{eq:dual}; the two averages agree as $\gamma\to1$ and are close at $\gamma=0.99$. The empirical $\lambda$ trajectories in Section~\ref{sec:exp}, which rise while the policy is over budget and then settle, are consistent with convergence to such a working multiplier.
\end{remark}

\begin{remark}[Budget feasibility and multiplier saturation]
\label{rem:saturation}
The bound $[\lambda_{\min},\lambda_{\max}]$ of Lemma~\ref{lem:bounded} trades guaranteed feasibility for stability. If the declared budget $d$ is feasible under the task reward, the dichotomy of~\eqref{eq:kkt} holds and the achieved average power settles at or below $d$. If $d$ is set below what the task admits, the dual ascent~\eqref{eq:dual} pushes $\lambda$ to its ceiling $\lambda_{\max}$, where it saturates: the constraint can no longer be tightened, and the policy settles at the lowest average power compatible with still performing the task, which may lie above $d$. The budget then acts less as a hard guarantee than as a monotone knob, since lowering $d$ monotonically lowers the achieved average power until saturation, and the gap between $d$ and the realized average power exposes, rather than hides, how aggressively energy is being traded against task accuracy. Section~\ref{sec:exp} reports this regime explicitly.
\end{remark}

\section{Vehicles, Tasks, and Evaluation}
\label{sec:setup}

\subsection{Simulator and Vehicles}
All experiments are conducted in MarineGym~\cite{chu2025marinegym}, a GPU-accelerated underwater-robotics RL platform that simulates rigid-body and hydrodynamic effects and thruster power for thousands of parallel environments. We study three vehicles with deliberately different actuation and inertia (Fig.~\ref{fig:robots}), so as to test whether the energy formulation generalizes across embodiments:
\begin{itemize}
\item \textbf{BlueROV}~\cite{bluerobotics_bluerov2}: a six-thruster multirotor-type ROV (Blue Robotics) that achieves full six-degree-of-freedom (6-DoF) control through differential thrust;
\item \textbf{BlueROV-Heavy}~\cite{bluerobotics_bluerov2_heavy}: an eight-thruster configuration of the same Blue Robotics platform, more strongly and redundantly actuated; and
\item \textbf{CUREE-AUV}~\cite{girdhar2023curee}: a compact, hovering-capable autonomous underwater vehicle developed for marine ecosystem exploration, driven by six thrusters and, as the heaviest platform here, the most demanding to actuate precisely.
\end{itemize}

\subsection{Tasks}
We evaluate four control tasks, illustrated in Fig.~\ref{fig:tasks}:
\begin{itemize}
\item \textbf{Hover}: reach and hold a 6-DoF setpoint from randomized initial conditions (Fig.~\ref{fig:tasks}a);
\item \textbf{Track-Lemniscate}: follow a 3-D figure-eight reference (Fig.~\ref{fig:tasks}b);
\item \textbf{Track-Circle}: follow a horizontal circular reference (Fig.~\ref{fig:tasks}c); and
\item \textbf{Track-Spiral}: follow a descending/ascending helical reference (Fig.~\ref{fig:tasks}d).
\end{itemize}
Together these cover station-keeping and tracking of references with increasing curvature and vertical motion, exercising different regions of each vehicle's actuation envelope.

\begin{figure}
    \centering
    \begin{subfigure}{0.32\linewidth}
        \centering
        \includegraphics[width=\linewidth]{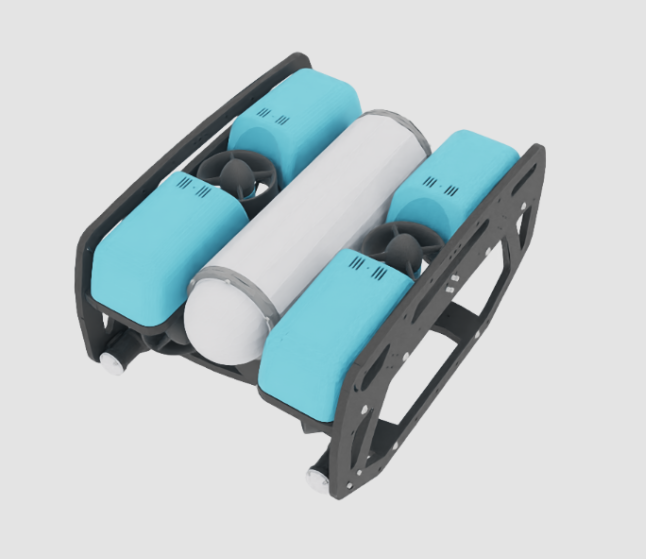}
        \caption{BlueROV}
        \label{fig:br}
    \end{subfigure}
    \hfill
    \begin{subfigure}{0.32\linewidth}
        \centering
        \includegraphics[width=\linewidth, trim={0pt 15pt 0pt 0pt}, clip]{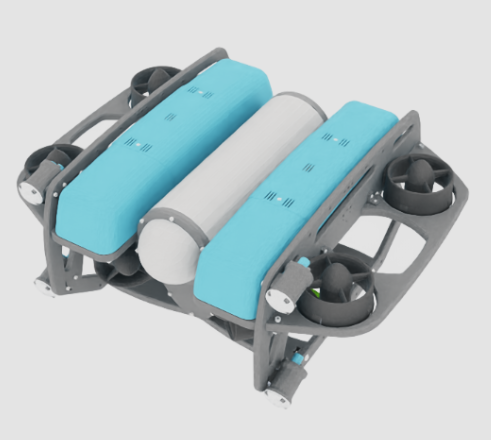}
        \caption{BlueROV-Heavy}
        \label{fig:brh}
    \end{subfigure}
    \hfill
    \begin{subfigure}{0.32\linewidth}
        \centering
        \includegraphics[width=\linewidth, height=0.1\textheight]{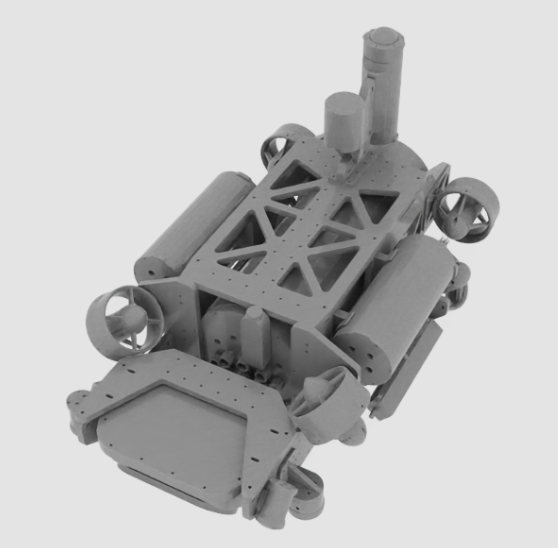}
        \caption{CUREE-AUV}
        \label{fig:CUREE-AUV}
    \end{subfigure}
    \caption{The three vehicles studied in MarineGym, spanning different actuation and inertia: (a) the six-thruster BlueROV, (b) the more strongly actuated eight-thruster BlueROV-Heavy, and (c) the heavier six-thruster CUREE-AUV.}
    \label{fig:robots}
\end{figure}

\begin{table}
\centering
\caption{Comparison of underwater robot platforms.}
\label{tab:robot_comparison}
\footnotesize
\begin{tabular}{lcccc}
\toprule
Robot & Weight  & Actuators & Volume  & CoB offset\\
& (kg) & &(m$^3$) & (m)\\
\midrule
BlueROV        & $11.2$ & T200 $\times$ 6 & $0.01135$  & $[0.0,0.0,0.01]$\\
BlueROV-Heavy   & $11.5$ & T200 $\times$ 8 & $0.01165$ & $[0.0,0.0,0.01]$\\
CUREE-AUV       & $22.7$ & T200 $\times$ 6 & $0.02275$ & $[-0.05,0.0,0.01]$\\
\bottomrule
\end{tabular}
\end{table}

\begin{figure}[tb]
    \centering
    \begin{subfigure}[t]{0.24\linewidth}
        \centering
        \includegraphics[width=\linewidth]{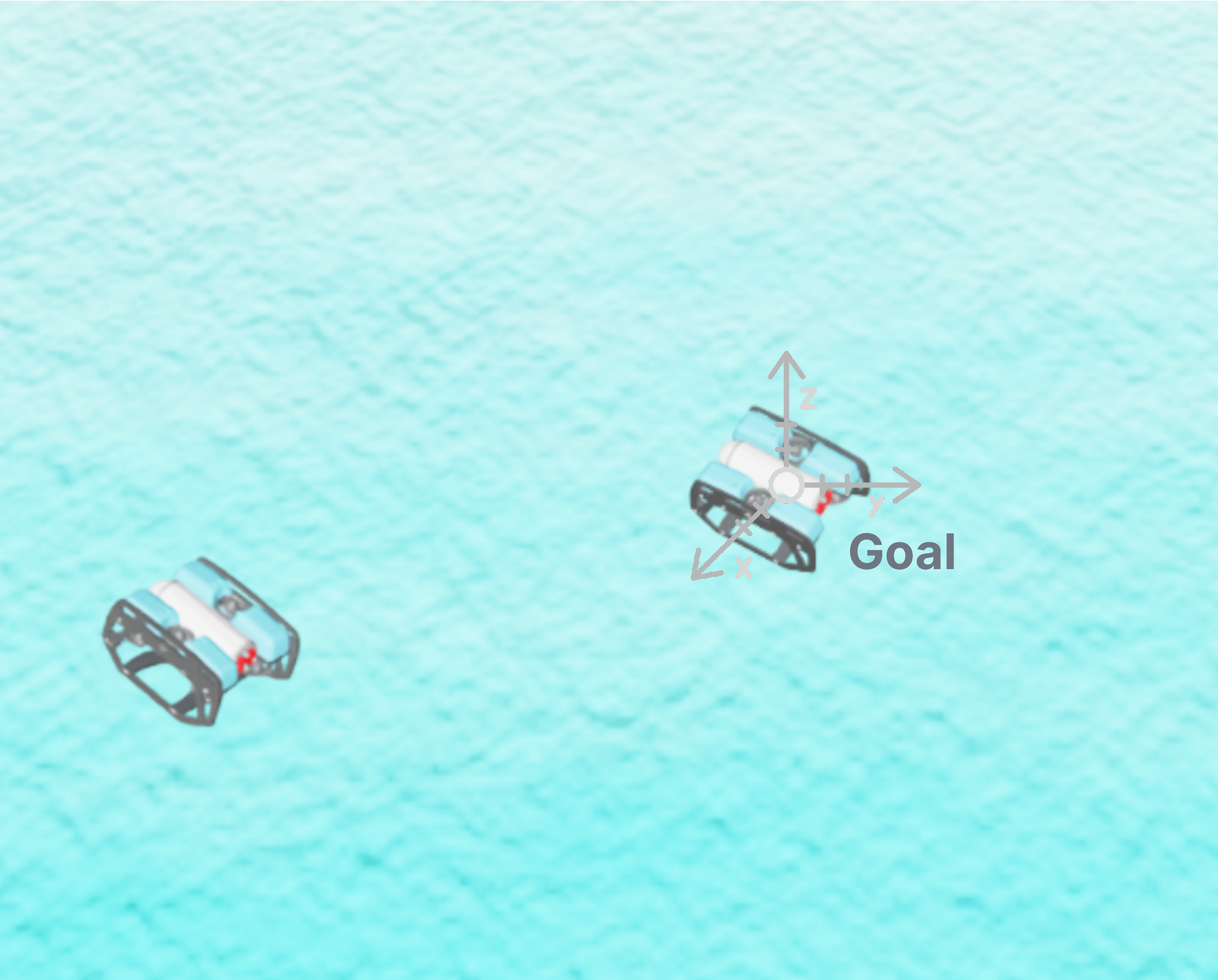}
        \caption{Hover}
        \label{fig:task-h}
    \end{subfigure}
    \hfill
    \begin{subfigure}[t]{0.24\linewidth}
        \centering
        \includegraphics[width=\linewidth, trim={0pt 60pt 0pt 60pt}, clip]{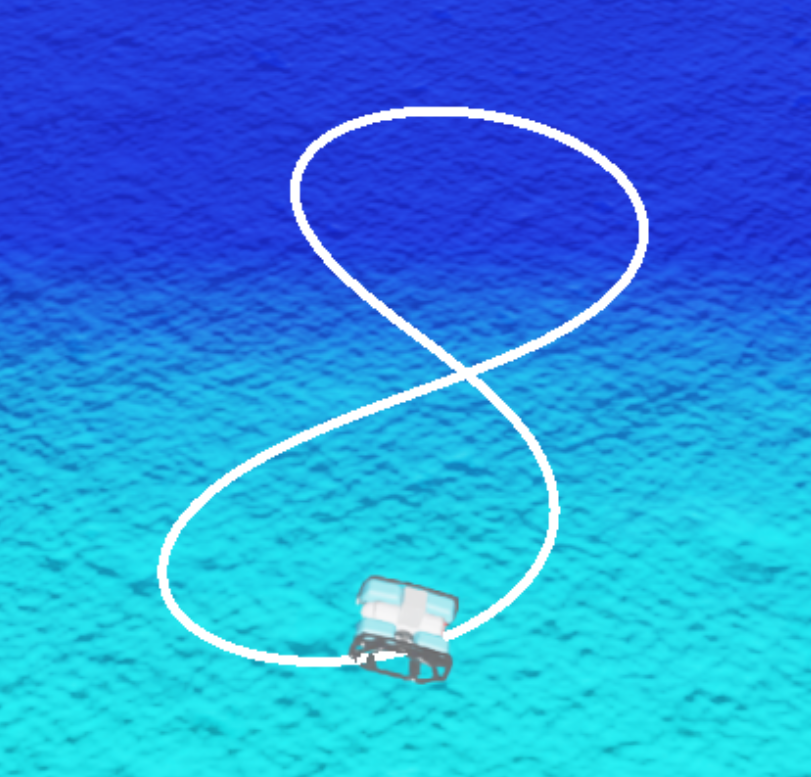}
        \caption{Lemniscate}
        \label{fig:task-l}
    \end{subfigure}
    \hfill
    \begin{subfigure}[t]{0.24\linewidth}
        \centering
        \includegraphics[width=\linewidth]{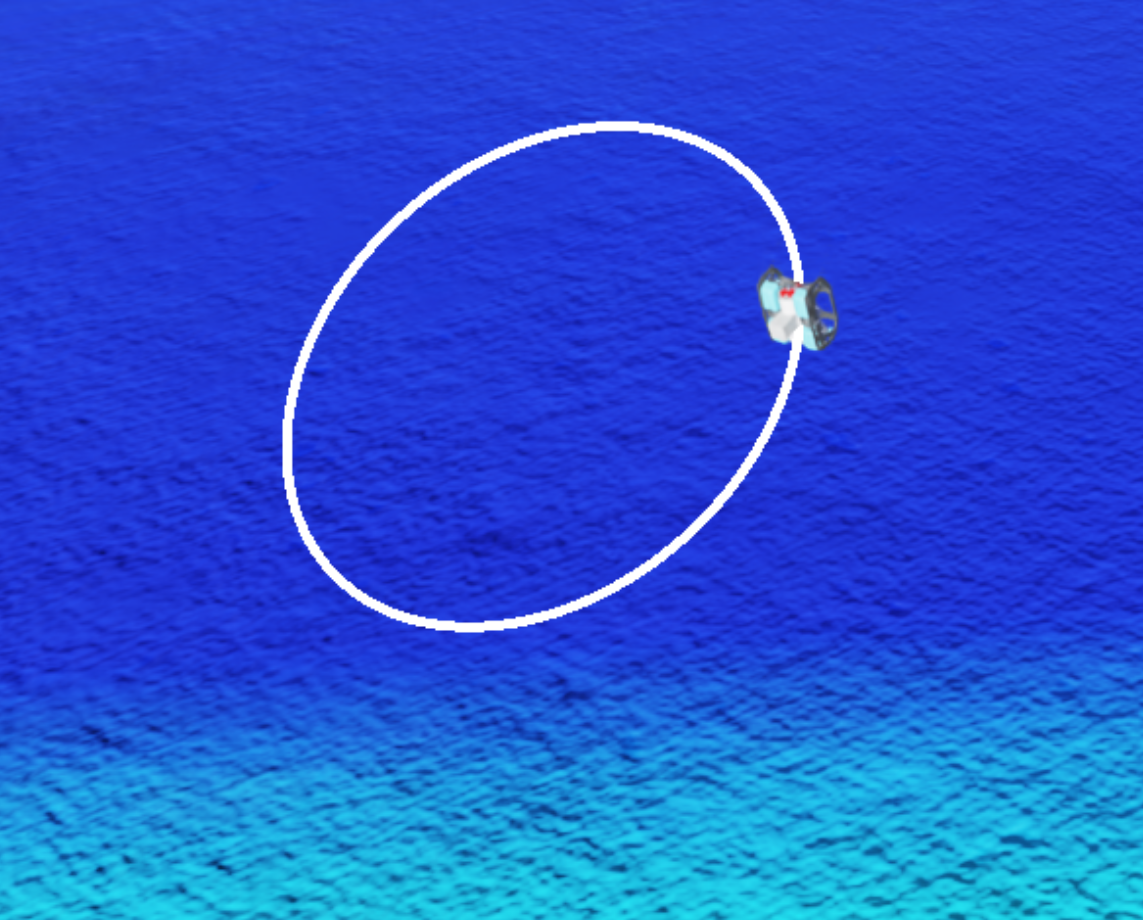}
        \caption{Circle}
        \label{fig:task-c}
    \end{subfigure}
    \hfill
    \begin{subfigure}[t]{0.24\linewidth}
        \centering
        \includegraphics[width=\linewidth]{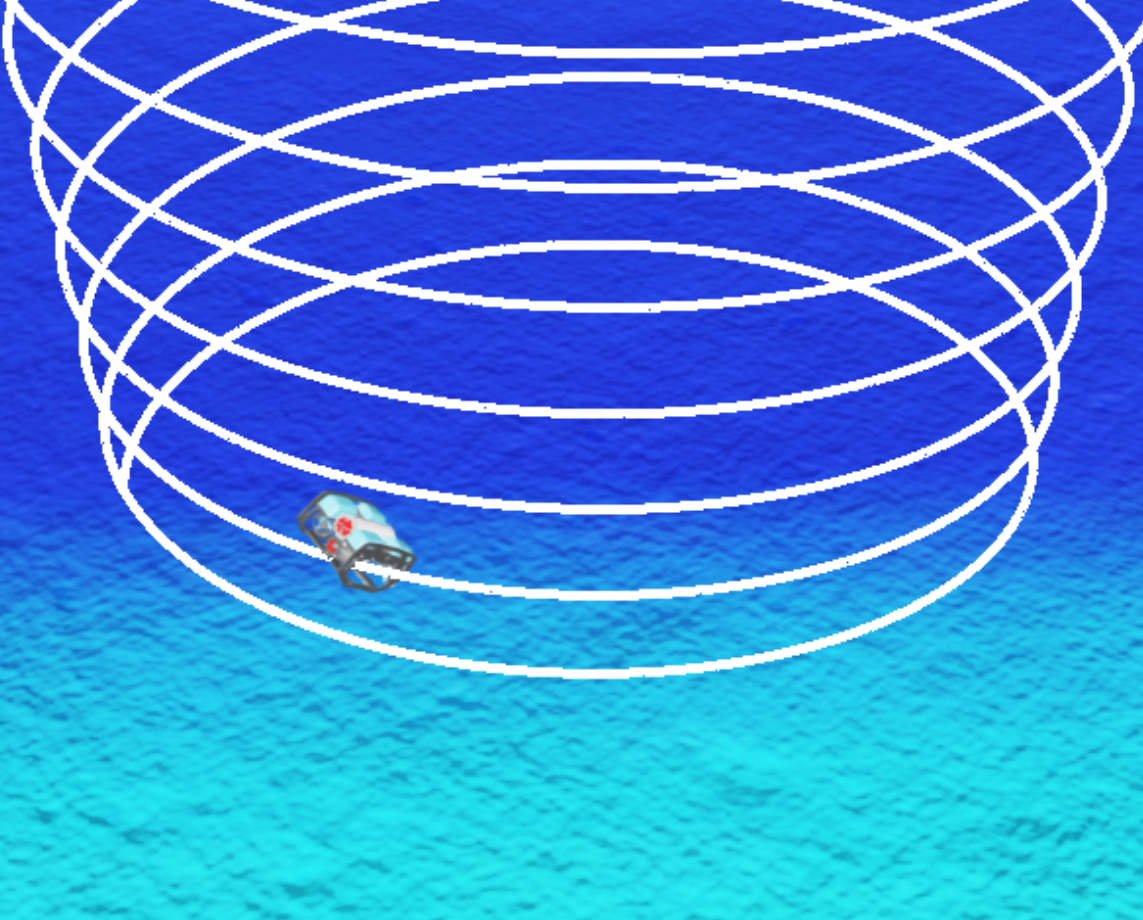}
        \caption{Spiral}
        \label{fig:task-s}
    \end{subfigure}
    \caption{The four control tasks shown with BlueROV in MarineGym: (a) station-keeping at a setpoint (``Goal'') and tracking a (b) lemniscate, (c) circular, and (d) spiral reference.}
    \label{fig:tasks}
\end{figure}

\subsection{Evaluation Metrics}
We report five deployment metrics. Avg. Power (W) is $\bar P$, the mean thruster power of~\eqref{eq:power} (lower is better), our primary efficiency metric. Smoothness is the action-smoothness score
\begin{equation}
\mathcal{S} = \frac{1}{T-1}\sum_{t=1}^{T-1}\bigl\|\boldsymbol{a}_t-\boldsymbol{a}_{t-1}\bigr\|_2 ,
\label{eq:smoothness}
\end{equation}
the mean $\ell_2$ norm of the change between consecutive throttle commands (closer to zero is smoother), a reported diagnostic rather than a constrained quantity. Track Err (m) is the root-mean-square error (RMSE) of the position relative to the reference (tracking tasks),
\begin{equation}
\mathrm{TrackErr} = \sqrt{\frac{1}{T}\sum_{t=0}^{T-1}\bigl\|\boldsymbol{p}_t-\boldsymbol{p}^{\mathrm{ref}}_t\bigr\|_2^2},
\label{eq:trackerr}
\end{equation}
with $\boldsymbol{p}_t$ the vehicle position and $\boldsymbol{p}^{\mathrm{ref}}_t$ the reference. TTG (time-to-goal) is the number of control steps to first reach and hold the setpoint tolerance (hover only). Success (\%) is the fraction of episodes meeting the task tolerance for the required duration. Episodic task return is monitored during training (Fig.~\ref{fig:curves}). Each policy is evaluated over 100 seeds per vehicle--task setting, and all numbers are reported as mean~$\pm$ standard deviation over these seeds.

\subsection{Baselines and Training}
We compare three policies that share the same actor and critic networks, observation space, task reward, optimizer, and number of training frames, and differ only in how energy enters the objective. (i) PPO is task-only, maximizing $J_r$ alone. (ii) PPO-Energy adds a fixed-weight action-effort penalty to the reward,
\begin{equation}
r_t^{\text{eff}} = r_t + w_E\, e^{-\lVert \boldsymbol{a}_t\rVert},\qquad w_E\ge 0,
\label{eq:energyreward}
\end{equation}
a unitless proxy that acts on action magnitude rather than power and whose weight $w_E$ is fixed by hand per setting. (iii) The proposed PPO-Lag of Section~\ref{sec:method} constrains the modeled average power to the budget $d$ through the dual update~\eqref{eq:dual}. Holding everything but the energy term identical isolates the effect of the energy formulation. The training was conducted on a mobile workstation equipped with an NVIDIA GeForce RTX~A4500 GPU (16~GB). Training uses $2{,}048$ parallel environments in MarineGym for $100$\,M frames per run, and every vehicle--task--method configuration is trained from three random seeds, under the domain randomization of Table~\ref{tab:disturbance_randomization}; the complete training and deployment procedure is given in Algorithm~\ref{alg:ppolag}, with hyperparameters in Table~\ref{tab:two_col_multirow}.

\begin{algorithm}[t]
\caption{Average-Power-Constrained PPO-Lagrangian: training and deployment}
\label{alg:ppolag}
\begin{algorithmic}[1]
\State \textbf{Hyperparameters:} budget $d$, dual step $\eta_\lambda$, clip $\epsilon$, GAE $\kappa$, entropy $c_e$, value weight $c_v$, epochs $E$, minibatches $B$, parallel envs $N$
\Statex \textit{// Training}
\State Initialize actor $\pi_\theta$, critics $V^r_\phi,V^c_\psi$, dual log-multiplier $\nu\gets\log\lambda_{\min}$
\For{iteration $=1,2,\dots$}
  \For{each of $N$ domain-randomized parallel envs (Table~\ref{tab:disturbance_randomization})}
    \State Roll out $\pi_\theta$: collect $\{(\boldsymbol{s}_t,\boldsymbol{a}_t,r_t,\boldsymbol{s}_{t+1})\}$, $\boldsymbol{a}_t\!\sim\!\pi_\theta(\cdot\!\mid\!\boldsymbol{s}_t)$
    \State $c_t\gets\sum_i P(f_{i,t})$ from the commanded thrusts \Comment{Eq.~\eqref{eq:t200}}
  \EndFor
  \State $\lambda\gets e^{\nu}$
  \State Reward GAE: $A^r_t,\hat{G}^r_t$ from $V^r_\phi$; \ \ cost GAE: $A^c_t,\hat{G}^c_t$ from $V^c_\psi,\,c_t\!-\!d$ \Comment{Eq.~\eqref{eq:gae}}
  \State $A_t(\lambda)\gets A^r_t-\lambda A^c_t$ \Comment{Eq.~\eqref{eq:surrogate}}
  \For{epoch $=1,\dots,E$, over minibatches of size $B$}
    \State Update $\theta,\phi,\psi$ by Adam on $\mathcal{J}(\theta,\phi,\psi)$ \Comment{Eqs.~\eqref{eq:clip},\eqref{eq:totalloss},\eqref{eq:valueloss}}
  \EndFor
  \State $\hat{J}_c\gets$ batch-mean average power $\tfrac{1}{T}\sum_t c_t$
  \State $\nu\gets\nu+\eta_\lambda(\hat{J}_c-d)$;\ \ $\lambda\gets\exp(\mathrm{clip}(\nu,\log\lambda_{\min},\log\lambda_{\max}))$ \Comment{Eqs.~\eqref{eq:dual},\eqref{eq:logspace}}
\EndFor
\Statex \textit{// Deployment (on-board inference)}
\State Load $\pi_\theta$; discard the critics and $\lambda$
\For{each control step $t$ on the vehicle}
  \State Observe $\boldsymbol{s}_t$; set $\boldsymbol{a}_t\gets$ mean of $\pi_\theta(\cdot\mid\boldsymbol{s}_t)$ (deterministic)
  \State Map $\boldsymbol{a}_t$ to thruster commands $\boldsymbol{u}$ and apply
\EndFor
\end{algorithmic}
\end{algorithm}

\begin{table}[t]
\centering
\caption{Disturbance and domain randomization settings used during training.}
\label{tab:disturbance_randomization}
\small
\begin{tabularx}{\linewidth}{lX}
\toprule
Parameter & Value \\
\midrule
\multicolumn{2}{c}{Disturbance Randomization} \\
\midrule
Max flow velocity & $[0.5, 0.5, 0.5, 0.0, 0.0, 0.0] m/s$\\
Flow velocity Gaussian noise & $[0.1, 0.1, 0.1, 0.0, 0.0, 0.0] m/s$\\
Payload mass & $[0.01, 0.2] \times Body Mass$\\
Payload $z$ position & $[-0.1, 0.1] m$\\
\midrule
\multicolumn{2}{c}{Robot Randomization} \\
\midrule
Body mass scale & $[0.8, 1.2]$ \\
Body volume scale & $[0.9, 1.1]$ \\
CoB offset scale & $[0.5, 1.5]$ \\
Inertia scale & $[0.8, 1.2]$ \\
Added mass scale & $[0.5, 1.0]$ \\
Linear damping scale & $[0.5, 1.0]$ \\
Quadratic damping scale & $[0.5, 1.0]$ \\
\midrule
\multicolumn{2}{c}{Actuator Randomization} \\
\midrule
Rotor time constants scale & $[0.8, 1.2]$ \\
Rotor force constants scale & $[0.8, 1.2]$ \\
\bottomrule
\end{tabularx}
\end{table}

\begin{table}[t]
\centering
\caption{PPO and Lagrangian-constraint hyperparameters.}
\label{tab:two_col_multirow}
\begin{tabular}{lm{0.5\linewidth}}
\toprule
\textbf{Component} & \textbf{Value} \\
\midrule
\multicolumn{2}{c}{PPO Hyperparameters} \\
\midrule
Actor & 3-layer MLP (256, 256, 256) \\
Critic & 3-layer MLP (256, 256, 256) \\
\multirow{3}{*}{Learning rate} & $1 \times 10^{-3}$ (BlueROV) \\
                               & $5 \times 10^{-4}$ (BlueROV-Heavy)\\ 
                               & $2 \times 10^{-3}$ (CUREE-AUV) \\
Training environments & $2,048$ \\
Training steps & $100,000,000$  \\
Random Training seeds & $3$ \\
\multirow{4}{*}{Episode length} & $200$ (Hover)\\
                                & $600$ (Track-Lemnicate)\\
                                & $600$ (Track-Circle)\\
                                & $600$ (Track-Spiral)\\
Batch size & $131,072$ ($2048 \times 64$) \\
Minibatch number & $16$ \\
PPO epochs & $4$ \\
Discount $\gamma$ & $0.99$\\
GAE parameter $\kappa$ & $0.95$ \\
Clip ratio $\epsilon$ & $0.1$ \\
Entropy coefficient $c_e$ & $0.001$ \\
Optimizer & Adam \\
Activation & ReLU \\
Simulation timestep & $0.016s$ \\
\midrule
\multicolumn{2}{c}{Lagrangian Constraint Hyperparameters} \\
\midrule
Initial Lagrange multiplier $\lambda_0$ & $e^{-2}$ \\
Multiplier range $[\lambda_{\min},\lambda_{\max}]$ & $[0.05, 2]$ \\
Dual step size $\eta_\lambda$ &$5\times 10^{-3}$\\
\multirow{3}{*}{Average-power budget $d$} & 1300\,W (BlueROV)\\
                                   & 1200\,W (BlueROV-Heavy) \\
                                   & 1300\,W (CUREE-AUV) \\
Cost-critic discount $\gamma_{c}$ & 0.99 \\
Cost GAE parameter $\kappa_{c}$ & 0.95 \\
\bottomrule
\end{tabular}
\end{table}


\begin{figure}[tb]
\centering
\begin{subfigure}{0.49\linewidth}
\centering
\includegraphics[width=\linewidth]{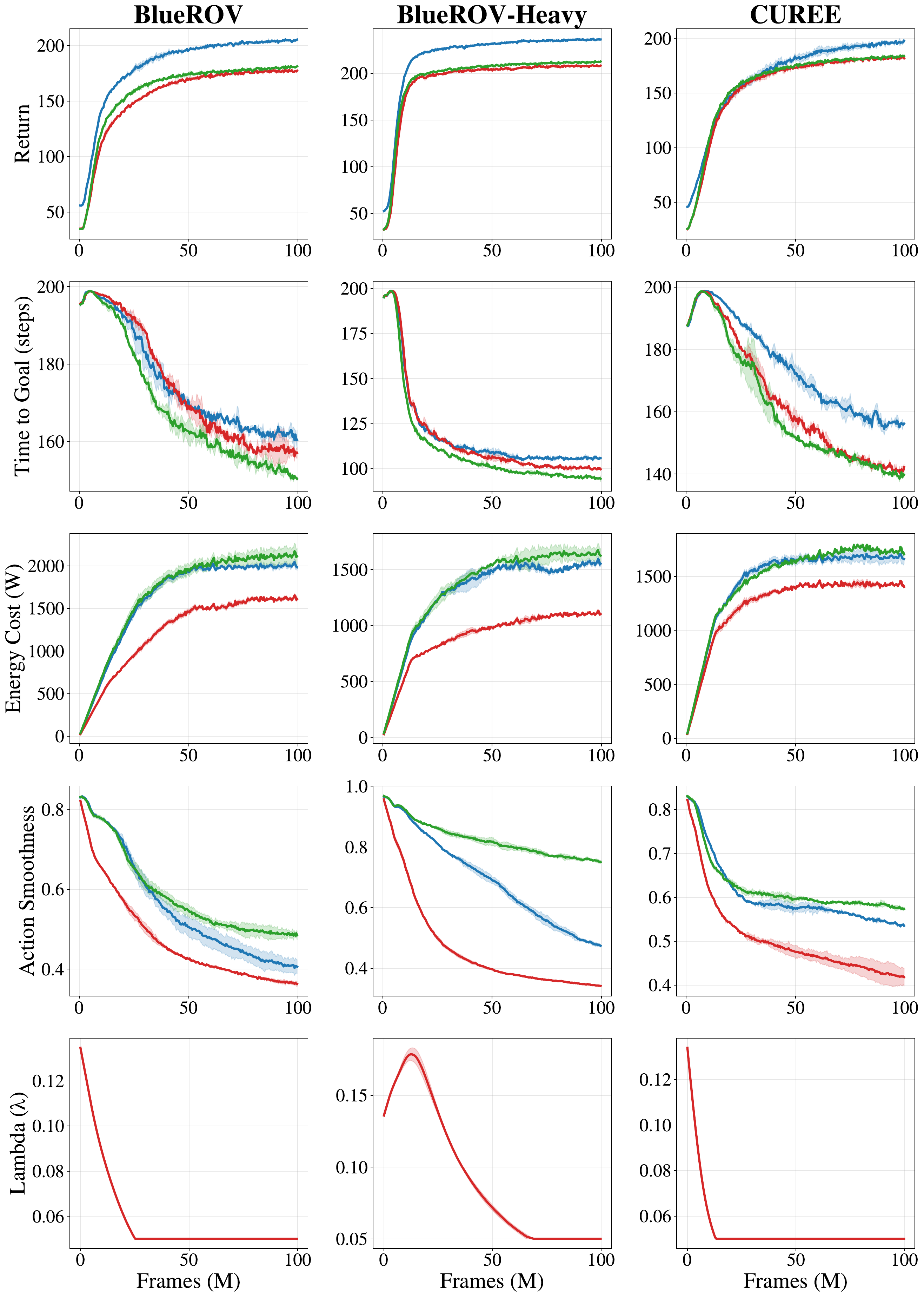}
\caption{Hover}
\end{subfigure}\hfill
\begin{subfigure}{0.49\linewidth}
\centering
\includegraphics[width=\linewidth]{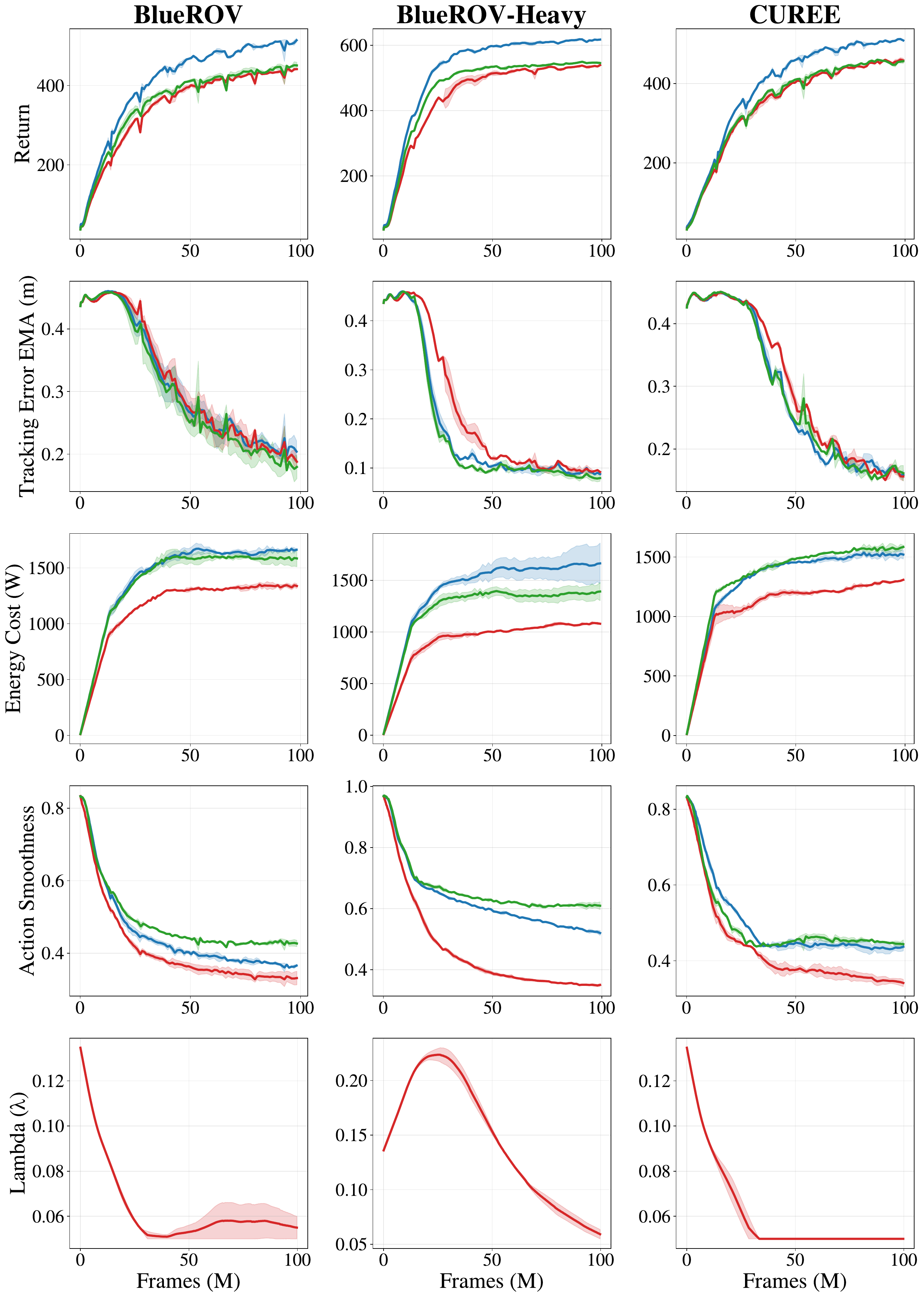}
\caption{Track-Lemniscate}
\end{subfigure}\\[2mm]
\begin{subfigure}{0.49\linewidth}
\centering
\includegraphics[width=\linewidth]{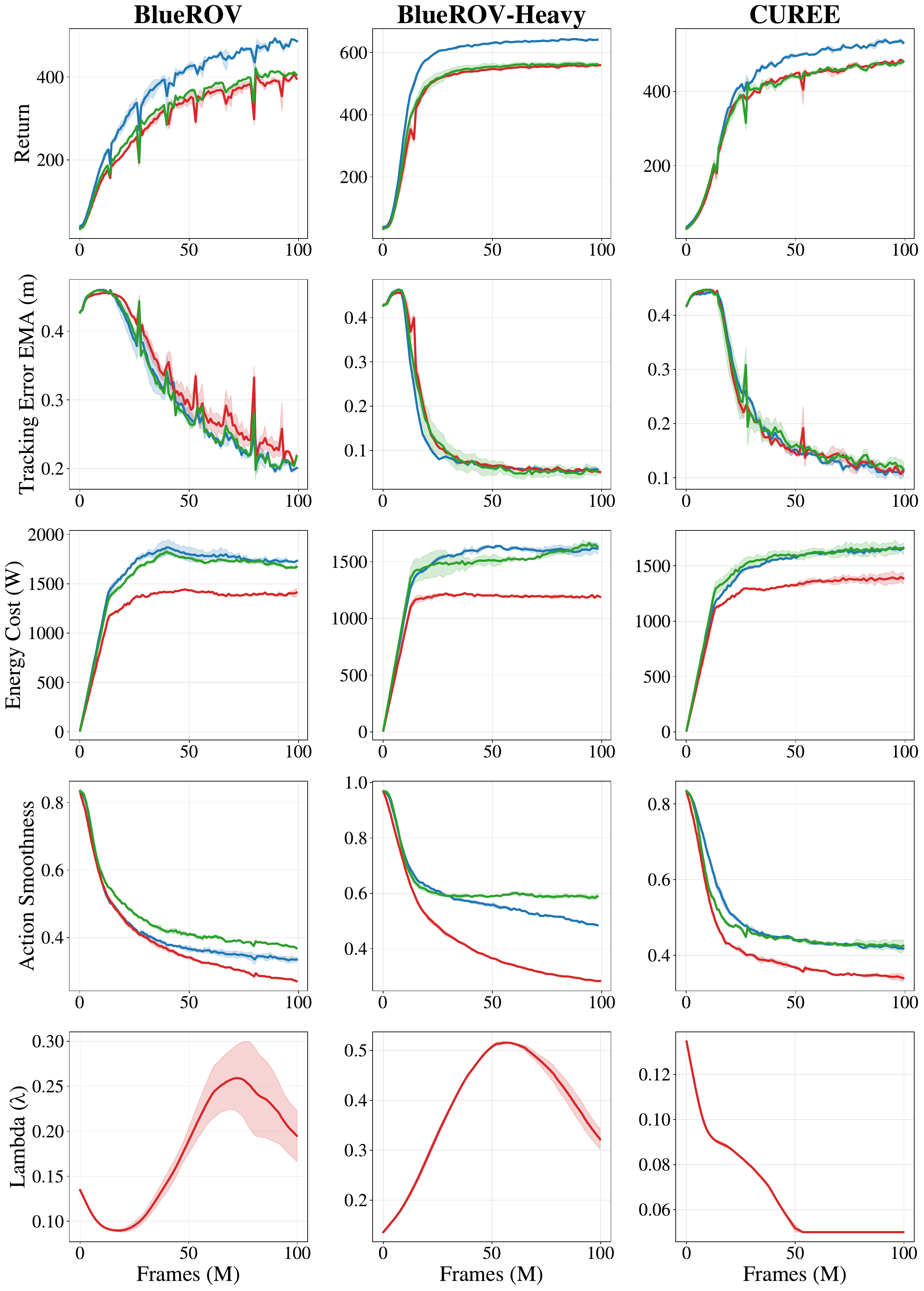}
\caption{Track-Circle}
\end{subfigure}\hfill
\begin{subfigure}{0.49\linewidth}
\centering
\includegraphics[width=\linewidth]{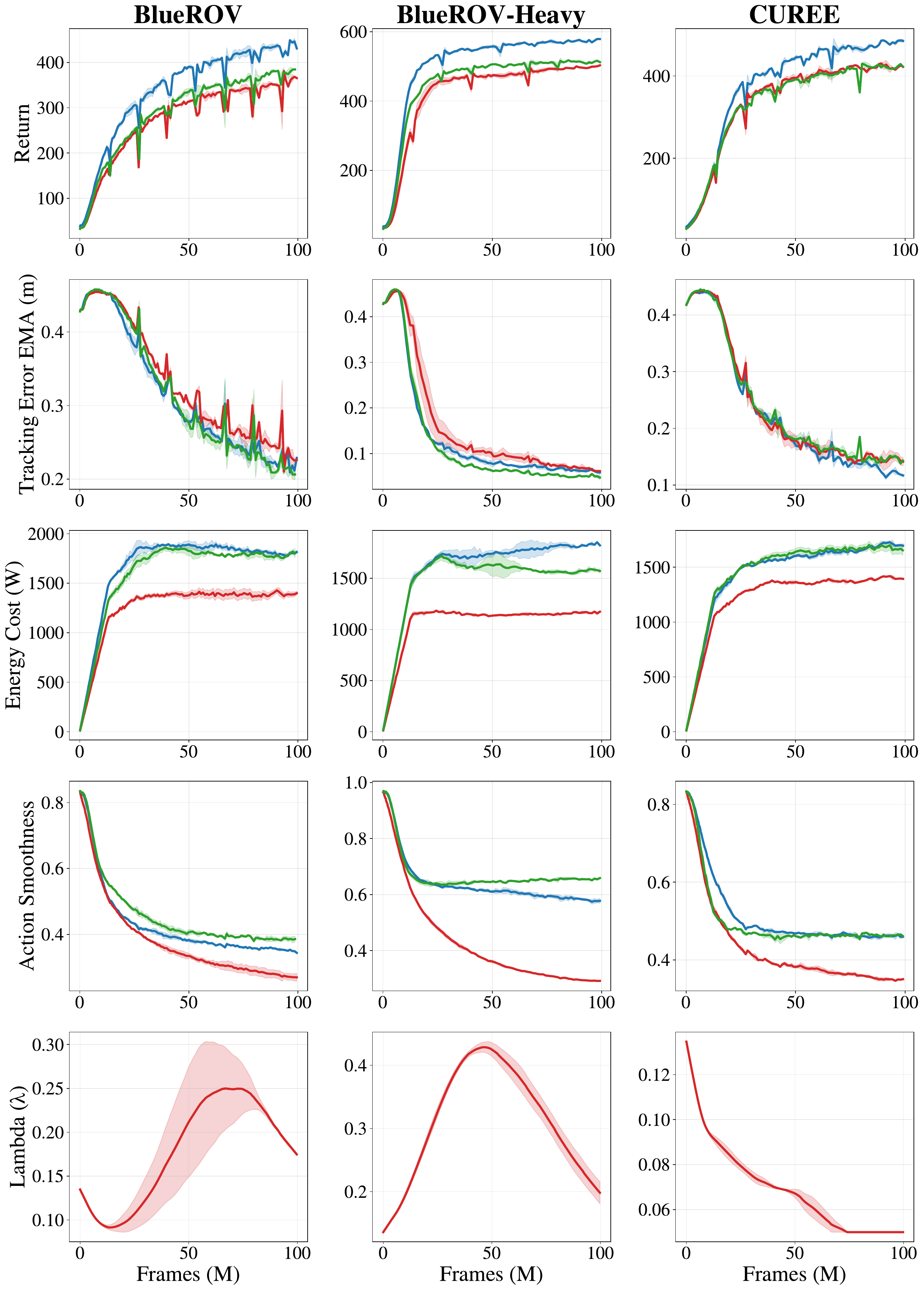}
\caption{Track-Spiral}
\end{subfigure}\\
\vspace{1mm}
\includegraphics[width=0.6\linewidth]{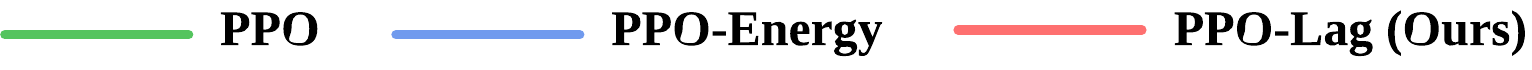}
\caption{Training curves for (a) Hover, (b) Track-Lemniscate, (c) Track-Circle, and (d) Track-Spiral. Columns: BlueROV, BlueROV-Heavy, CUREE-AUV; rows: return, time-to-goal (Hover) or tracking-error RMSE (tracking), average power, smoothness, and the multiplier $\lambda$. Each curve is the mean over three training seeds, and shaded bands show their spread.}
\label{fig:curves}
\end{figure}

\section{Experiments and Analysis}
\label{sec:exp}

We organize the analysis around five research questions: (RQ1) How do the three energy formulations behave during training? (RQ2) How much average power does each draw at deployment? (RQ3) How smooth is the resulting actuation? (RQ4) Is task quality (success and tracking) preserved? (RQ5) What is the controller's on-board deployment cost on resource-constrained hardware? Table~\ref{tab:performance} reports the full quantitative comparison; the figures provide the supporting training and rollout evidence.

\begin{table*}[tb]
\centering
\small
\setlength{\tabcolsep}{3pt}
\caption{Performance across tasks, vehicles, and methods. In each task--vehicle group, the best of the three methods per column is in \textbf{bold} (lowest for average power, smoothness, tracking error, and time-to-goal; highest for success). Values are mean~$\pm$ std over 100 evaluation seeds.}
\label{tab:performance}
\resizebox{\textwidth}{!}{%
\begin{tabular}{lllccccc}
\toprule
Task & Robot & Algorithm & Smoothness & Track Err & TTG & Avg. Power (W) & Success \\
\midrule
Hover & BlueROV & PPO & $0.280\pm0.129^{**}$ & -- & $\mathbf{153.3\pm34.1}$ & $2196.1\pm244.8^{**}$ & \textbf{80.0\%} \\
Hover & BlueROV & PPO-Energy & $0.185\pm0.103^{**}$ & -- & $157.7\pm36.8$ & $1935.0\pm173.6^{**}$ & 72.0\% \\
Hover & BlueROV & PPO-Lag (Ours) & $\mathbf{0.145\pm0.107}$ & -- & $153.4\pm35.9$ & $\mathbf{1636.1\pm342.9}$ & 79.0\% \\
\addlinespace

Hover & BlueROV-Heavy & PPO & $0.081\pm0.059^{**}$ & -- & $\mathbf{95.4\pm24.5}$ & $1611.6\pm329.0^{**}$ & \textbf{100.0\%} \\
Hover & BlueROV-Heavy & PPO-Energy & $0.129\pm0.105^{**}$ & -- & $107.2\pm28.8^{*}$ & $1544.2\pm271.7^{**}$ & 99.0\% \\
Hover & BlueROV-Heavy & PPO-Lag (Ours) & $\mathbf{0.021\pm0.005}$ & -- & $98.2\pm24.0$ & $\mathbf{1103.0\pm276.8}$ & 100.0\% \\
\addlinespace
Hover & CUREE-AUV & PPO & $0.251\pm0.127^{**}$ & -- & $138.5\pm35.8$ & $1670.0\pm300.0^{**}$ & $87.0\%$ \\
Hover & CUREE-AUV & PPO-Energy & $0.261\pm0.137^{**}$ & -- & $156.2\pm37.3^{**}$ & $1650.7\pm280.4^{**}$ & $73.0\%$ \\
Hover & CUREE-AUV & PPO-Lag (Ours) & $\mathbf{0.084\pm0.047}$ & -- & $\mathbf{133.5\pm37.9}$ & $\mathbf{1416.2\pm319.1}$ & \textbf{89.0\%} \\

\midrule
Track-Lemniscate & BlueROV & PPO & $0.145\pm0.089^{**}$ & $0.383\pm0.559$ & -- & $2054.1\pm170.8^{**}$ & 71.0\% \\
Track-Lemniscate & BlueROV & PPO-Energy & $0.113\pm0.063^{**}$ & $0.407\pm0.709$ & -- & $2027.1\pm132.4^{**}$ & 72.0\% \\
Track-Lemniscate & BlueROV & PPO-Lag (Ours) & $\mathbf{0.091\pm0.029}$ & $\mathbf{0.241\pm0.452}$ & -- & $\mathbf{1664.4\pm164.1}$ & \textbf{92.0\%} \\
\addlinespace
Track-Lemniscate & BlueROV-Heavy & PPO & $0.200\pm0.074^{**}$ & $\mathbf{0.060\pm0.056}^{**}$ & -- & $1893.2\pm154.8^{**}$ & \textbf{100.0\%} \\
Track-Lemniscate & BlueROV-Heavy & PPO-Energy & $\mathbf{0.131\pm0.052}^{**}$ & $0.089\pm0.050$ & -- & $2165.2\pm139.2^{**}$ & 100.0\% \\
Track-Lemniscate & BlueROV-Heavy & PPO-Lag (Ours) & $0.158\pm0.054$ & $0.090\pm0.070$ & -- & $\mathbf{1448.2\pm142.4}$ & 100.0\% \\
\addlinespace
Track-Lemniscate & CUREE-AUV & PPO & $0.120\pm0.056^{**}$ & $\mathbf{0.711\pm1.098}$ & -- & $1829.2\pm118.7^{**}$ & \textbf{59.0\%} \\
Track-Lemniscate & CUREE-AUV & PPO-Energy & $0.127\pm0.069^{**}$ & $1.186\pm1.518$ & -- & $1731.5\pm125.5^{**}$ & $43.0\%$ \\
Track-Lemniscate & CUREE-AUV & PPO-Lag (Ours) & $\mathbf{0.091\pm0.044}$ & $0.850\pm1.332$ & -- & $\mathbf{1577.6\pm150.2}$ & $58.0\%$ \\
\midrule
Track-Circle & BlueROV & PPO & $0.049\pm0.048$ & $3.027\pm2.345^{**}$ & -- & $2255.4\pm210.6^{**}$ & 35.0\% \\
Track-Circle & BlueROV & PPO-Energy & $0.068\pm0.055^{**}$ & $1.542\pm1.636$ & -- & $2120.9\pm227.2^{**}$ & \textbf{48.0\%} \\
Track-Circle & BlueROV & PPO-Lag (Ours) & $\mathbf{0.047\pm0.039}$ & $\mathbf{1.533\pm1.620}$ & -- & $\mathbf{1746.2\pm343.9}$ & 45.0\% \\
\addlinespace
Track-Circle & BlueROV-Heavy & PPO & $0.179\pm0.079^{**}$ & $\mathbf{0.349\pm2.039}$ & -- & $2117.5\pm175.9^{**}$ & \textbf{98.0\%} \\
Track-Circle & BlueROV-Heavy & PPO-Energy & $0.164\pm0.075^{**}$ & $0.371\pm1.617$ & -- & $2315.7\pm157.7^{**}$ & 97.0\% \\
Track-Circle & BlueROV-Heavy & PPO-Lag (Ours) & $\mathbf{0.093\pm0.046}$ & $0.469\pm1.728$ & -- & $\mathbf{1720.0\pm177.5}$ & 95.0\% \\
\addlinespace
Track-Circle & CUREE-AUV & PPO & $0.232\pm0.109^{**}$ & $\mathbf{1.044\pm2.350}$ & -- & $1788.3\pm127.1^{**}$ & \textbf{80.0\%} \\
Track-Circle & CUREE-AUV & PPO-Energy & $0.156\pm0.119$ & $2.181\pm2.781^{**}$ & -- & $1810.2\pm135.4^{**}$ & $58.0\%$ \\
Track-Circle & CUREE-AUV & PPO-Lag (Ours) & $\mathbf{0.137\pm0.071}$ & $1.216\pm2.032$ & -- & $\mathbf{1469.5\pm181.1}$ & $69.0\%$ \\
\midrule
Track-Spiral & BlueROV & PPO & $0.050\pm0.046^{**}$ & $1.778\pm2.284$ & -- & $2140.9\pm217.4^{**}$ & 52.0\% \\
Track-Spiral & BlueROV & PPO-Energy & $0.064\pm0.054^{**}$ & $\mathbf{1.546\pm1.953}$ & -- & $2212.1\pm177.0^{**}$ & 47.0\% \\
Track-Spiral & BlueROV & PPO-Lag (Ours) & $\mathbf{0.034\pm0.031}$ & $1.835\pm2.528$ & -- & $\mathbf{1558.0\pm395.7}$ & \textbf{59.0\%} \\
\addlinespace
Track-Spiral & BlueROV-Heavy & PPO & $0.220\pm0.102^{**}$ & $\mathbf{0.090\pm0.448}$ & -- & $1912.0\pm170.3^{**}$ & 99.0\% \\
Track-Spiral & BlueROV-Heavy & PPO-Energy & $0.178\pm0.108^{**}$ & $0.104\pm0.078$ & -- & $2243.2\pm125.8^{**}$ & \textbf{100.0\%} \\
Track-Spiral & BlueROV-Heavy & PPO-Lag (Ours) & $\mathbf{0.079\pm0.032}$ & $0.127\pm0.129$ & -- & $\mathbf{1644.6\pm128.3}$ & 100.0\% \\
\addlinespace
Track-Spiral & CUREE-AUV & PPO & $0.186\pm0.137^{**}$ & $\mathbf{1.876\pm2.282}$ & -- & $1873.4\pm145.2^{**}$ & \textbf{74.0\%} \\
Track-Spiral & CUREE-AUV & PPO-Energy & $0.211\pm0.150^{**}$ & $2.109\pm3.098$ & -- & $1857.4\pm117.9^{**}$ & $65.0\%$ \\
Track-Spiral & CUREE-AUV & PPO-Lag (Ours) & $\mathbf{0.126\pm0.088}$ & $2.098\pm2.943$ & -- & $\mathbf{1565.4\pm217.8}$ & $66.0\%$ \\
\bottomrule
\end{tabular}
}

\vspace{2pt}
\footnotesize
$^{*}$ $p<0.05$, $^{**}$ $p<0.01$ (Welch's $t$-test, compared PPO and PPO-Energy to PPO-Lag).

\end{table*}

\subsection{RQ1: Training Dynamics}
Fig.~\ref{fig:curves} plots, for each task and vehicle (one column per vehicle), the training curves of return, time-to-goal or tracking error, average power, action smoothness, and, for PPO-Lag, the Lagrange multiplier $\lambda$. Three patterns recur. First, all three methods converge to comparable task return, so constraining the average power does not impede task learning; in several settings, PPO-Lag even finishes marginally higher (e.g., BlueROV lemniscate in Table~\ref{tab:performance}). Second, the PPO-Lag average-power curve (red) separates downward from PPO (green) and PPO-Energy (blue) early in training and the gap persists to convergence, while its smoothness curve likewise settles closest to zero. Third, the $\lambda$ curve rises through early-to-mid training as the policy initially overspends the budget, then decreases and plateaus once the average power has been driven down, the signature of a dual variable settling at a working value; on the tightest-budget vehicles it plateaus near its ceiling $\lambda_{\max}$, consistent with the residual over-budget gap quantified in RQ2 below. This online adaptation is exactly what a fixed energy weight cannot provide.

\subsection{RQ2: Average-Power Efficiency}
PPO-Lag draws the lowest mean power in all twelve vehicle--task settings (Table~\ref{tab:performance}, bold), and every mean-power gap over both baselines is significant at $p<0.01$ (Welch's $t$-test). Relative to the task-only PPO, the reduction spans $14\%$ (BlueROV-Heavy spiral, $1644.6$ vs.\ $1912.0$\,W) to $32\%$ (BlueROV-Heavy hover, $1103.0$ vs.\ $1611.6$\,W) and averages about $20\%$, with the largest reductions on the strongly actuated BlueROV-Heavy and BlueROV (e.g., BlueROV spiral $1558.0$ vs.\ $2140.9$\,W, a $27\%$ reduction). Against the energy-reward PPO-Energy, the mean power reduction is lower everywhere as well, by $9$--$33\%$. The per-rollout profiles in Fig.~\ref{fig:power} match the table: PPO-Lag (red) draws the least power throughout the episode, separating from the other two right after the initial transient.

The realized average power is not held strictly to the declared budget. We set deliberately tight budgets ($d=1300/1200/1300$\,W for BlueROV/BlueROV-Heavy/CUREE-AUV), and the mean power lies above $d$ in all but one setting (BlueROV-Heavy hover, $1103$\,W). Under such budgets the multiplier saturates at $\lambda_{\max}$ (Remark~\ref{rem:saturation}), so the policy converges to the lowest average power compatible with the task rather than to $d$ exactly. The budget therefore acts as the monotone knob of Remark~\ref{rem:saturation}: lowering $d$ presses average power down without a per-setting weight search, and the gap to $d$ measures how hard energy is being traded against the task.

Crucially, the energy-reward baseline is not a reliable substitute. In $5$ of the $12$ settings PPO-Energy actually draws more average power than the task-only PPO, for example BlueROV-Heavy spiral ($2243.2$ vs.\ $1912.0$\,W, $+17\%$), BlueROV-Heavy lemniscate ($2165.2$ vs.\ $1893.2$\,W, $+14\%$), and BlueROV-Heavy circle ($2315.7$ vs.\ $2117.5$\,W, $+9\%$). Two factors explain this: the effort reward acts on action magnitude rather than modeled power, so shrinking it need not reduce average power; and one fixed weight cannot match the differing power--accuracy trade-offs across vehicles and tasks. PPO-Lag, acting on modeled power through an adaptive multiplier, uses less average power than both alternatives in every setting, the central empirical result of this paper.

\begin{figure}[tb]
    \centering
    \begin{subfigure}{0.49\linewidth}
        \centering
        \includegraphics[width=\linewidth]{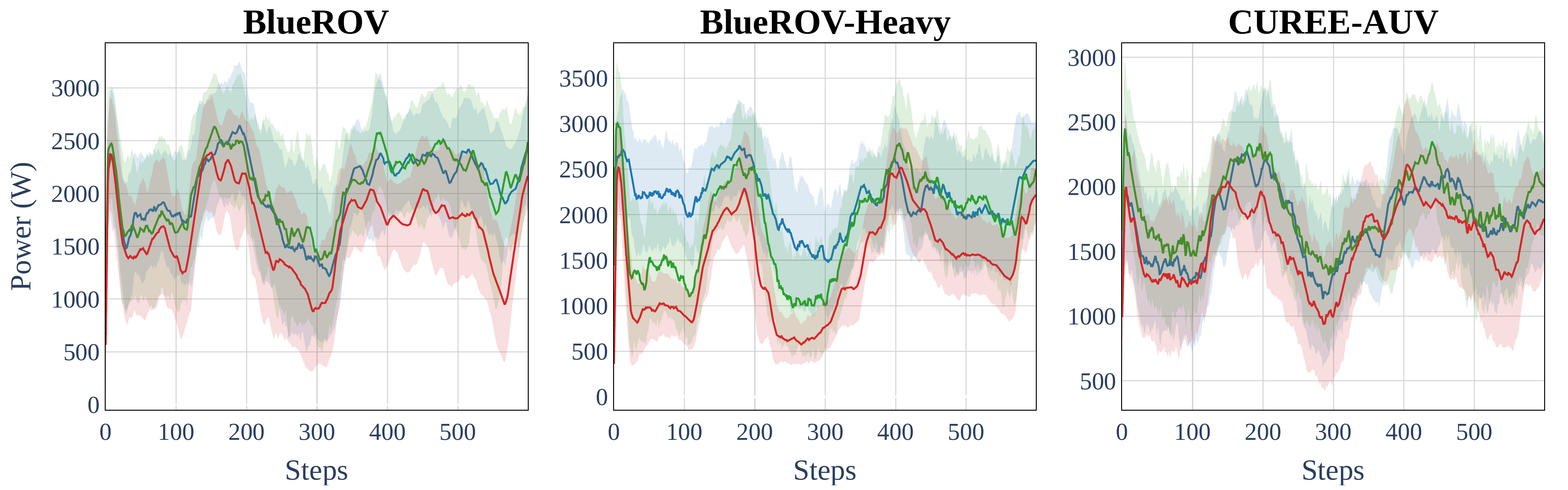}
        \caption{Track-Lemniscate}
    \end{subfigure}\hfill
    \begin{subfigure}{0.49\linewidth}
        \centering
        \includegraphics[width=\linewidth]{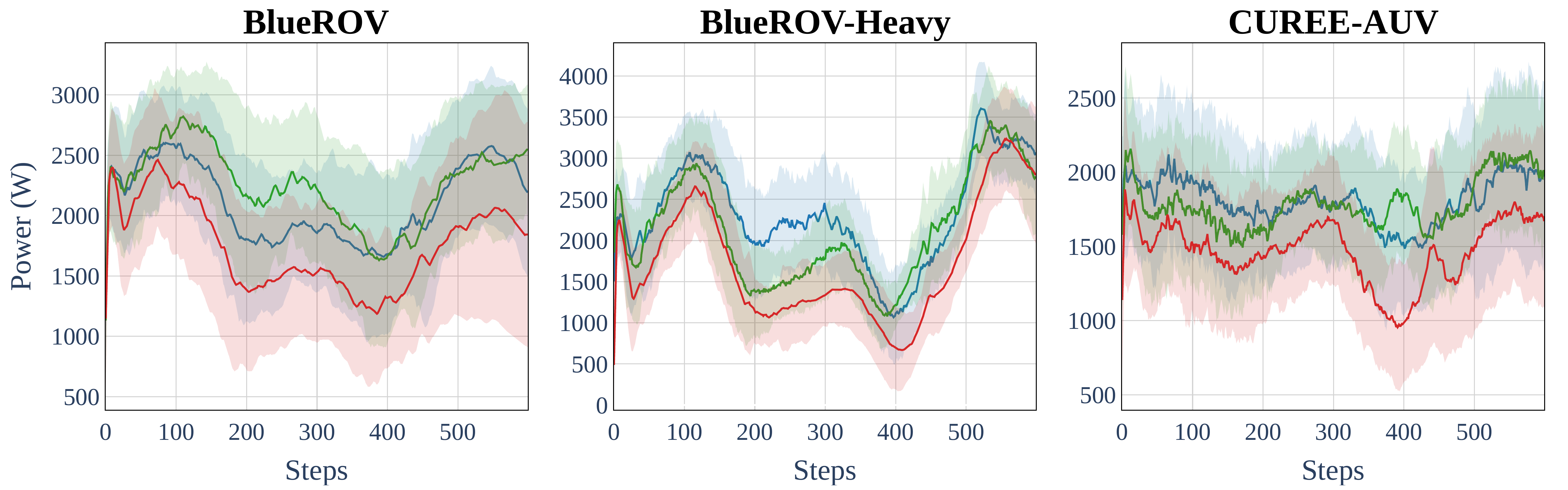}
        \caption{Track-Circle}
    \end{subfigure}\\[2mm]
    \begin{subfigure}{0.49\linewidth}
        \centering
        \includegraphics[width=\linewidth]{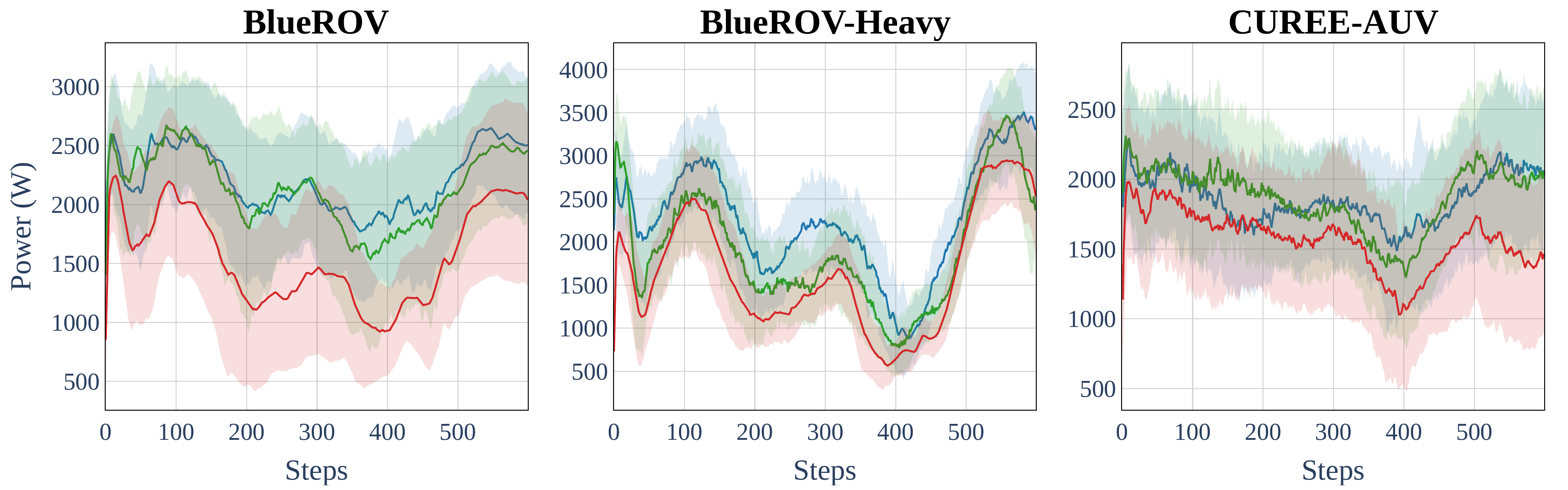}
        \caption{Track-Spiral}
    \end{subfigure}\hfill
    \begin{subfigure}{0.49\linewidth}
        \centering
        \includegraphics[width=\linewidth]{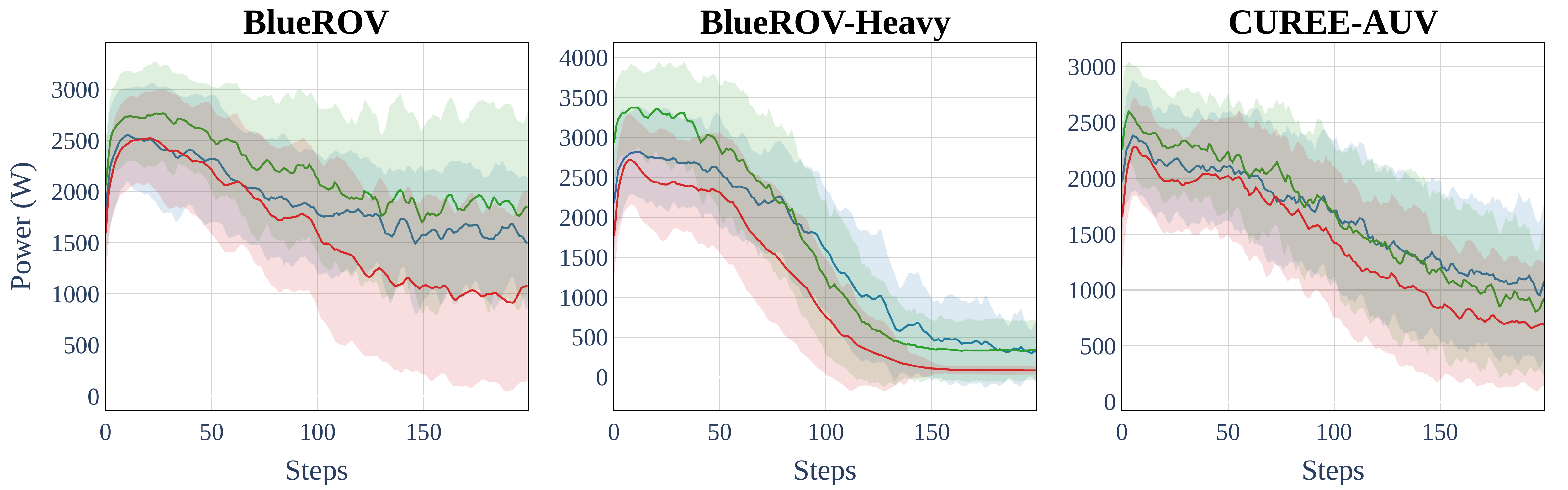}
        \caption{Hover}
    \end{subfigure}\\
    \vspace{1mm}
    \includegraphics[width=0.6\linewidth]{figures/legend.pdf}
    \caption{Mean thruster power over an evaluation episode, per vehicle (columns); each curve is the mean over 100 seeds, with shaded bands denoting $\pm1$ standard deviation.}
    \label{fig:power}
\end{figure}

\subsection{RQ3: Control Smoothness}
PPO-Lag attains the best (closest-to-zero) smoothness score in eleven of the twelve settings (Table~\ref{tab:performance}), and in most of them the margin over both baselines is significant at $p<0.01$. The gains are largest on hover and on the strongly actuated vehicles: BlueROV-Heavy hover improves to $0.021$ from $0.081$ (PPO) and $0.129$ (PPO-Energy), CUREE-AUV hover to $0.084$ from $0.251$ and $0.261$, and BlueROV-Heavy spiral to $0.079$ from $0.220$ and $0.178$. The lone exception is BlueROV-Heavy lemniscate, where PPO-Energy is marginally smoother ($0.131$ vs.\ $0.158$) and PPO-Lag ranks a close second. The per-rollout traces in Fig.~\ref{fig:smoothness} show the same ordering. Smoother commands are operationally valuable underwater, reducing thruster wear, mechanical fatigue.

\begin{figure}[tb]
    \centering
    \begin{subfigure}{0.49\linewidth}
        \centering
        \includegraphics[width=\linewidth]{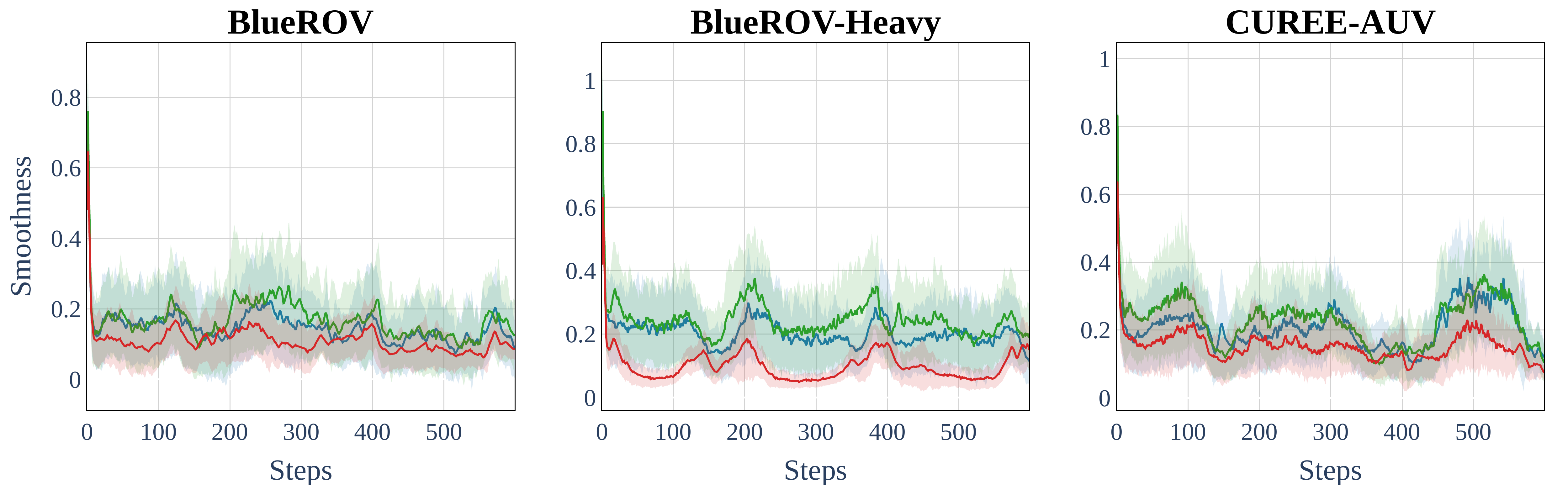}
        \caption{Track-Lemniscate}
    \end{subfigure}\hfill
    \begin{subfigure}{0.49\linewidth}
        \centering
        \includegraphics[width=\linewidth]{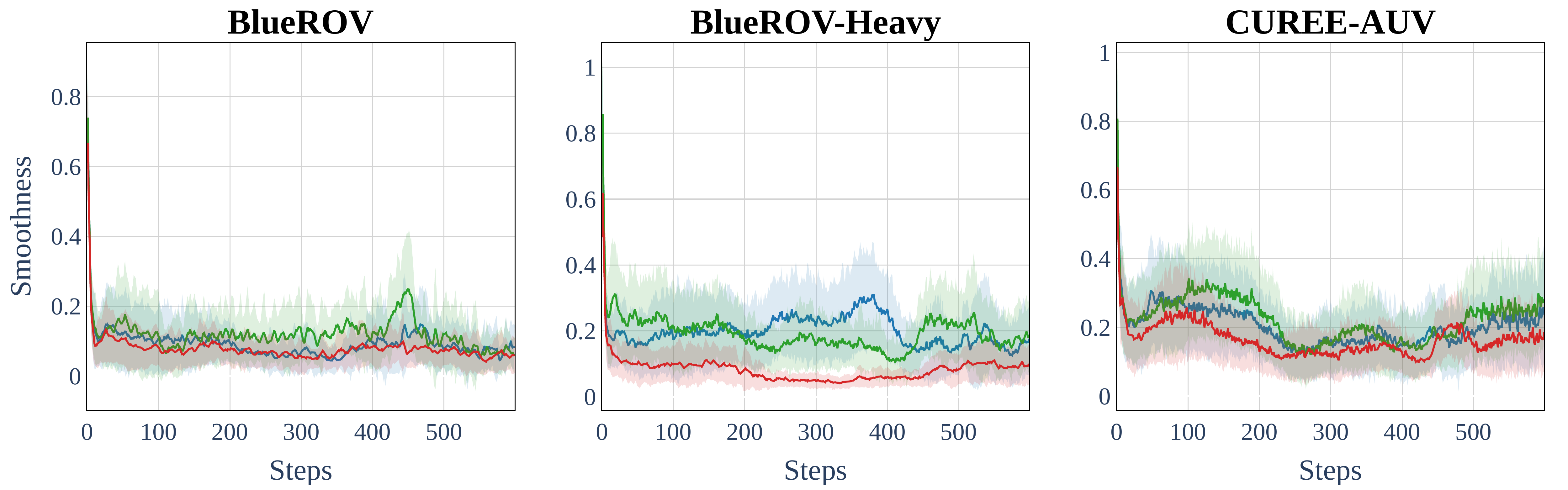}
        \caption{Track-Circle}
    \end{subfigure}\\[2mm]
    \begin{subfigure}{0.49\linewidth}
        \centering
        \includegraphics[width=\linewidth]{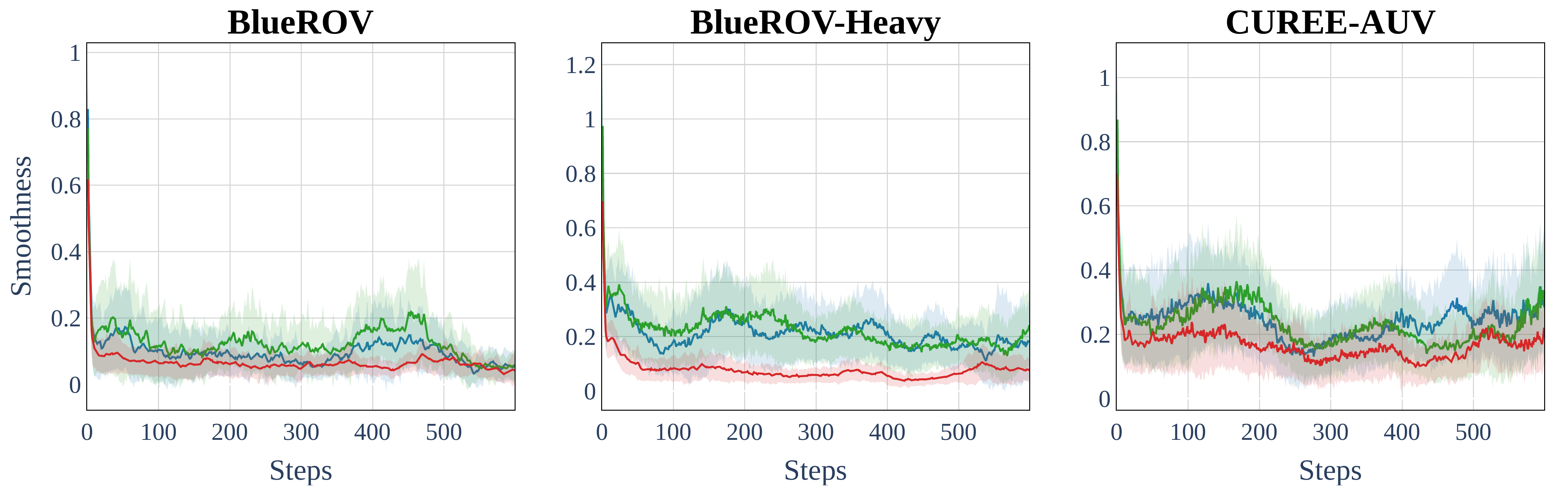}
        \caption{Track-Spiral}
    \end{subfigure}\hfill
    \begin{subfigure}{0.49\linewidth}
        \centering
        \includegraphics[width=\linewidth]{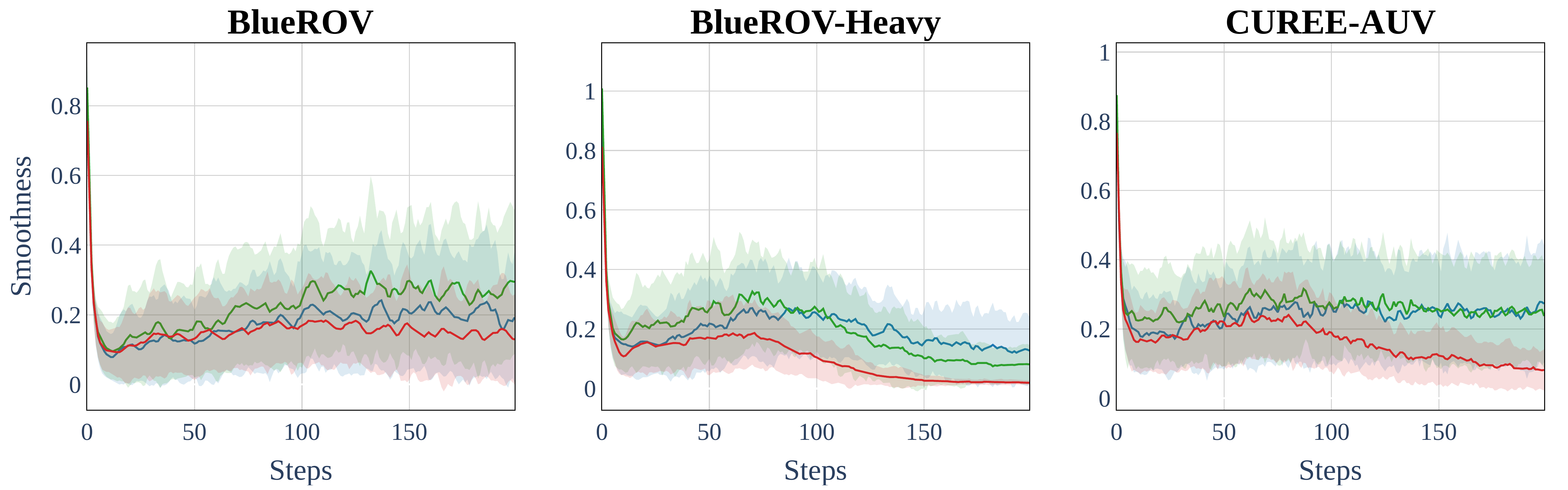}
        \caption{Hover}
    \end{subfigure}\\
    \vspace{1mm}
    \includegraphics[width=0.6\linewidth]{figures/legend.pdf}
    \caption{Action smoothness over an evaluation episode, per vehicle (columns); each curve is the mean over 100 seeds.}
    \label{fig:smoothness}
\end{figure}

\begin{figure}[tb]
    \centering
    \begin{subfigure}{0.49\linewidth}
        \centering
        \includegraphics[width=\linewidth]{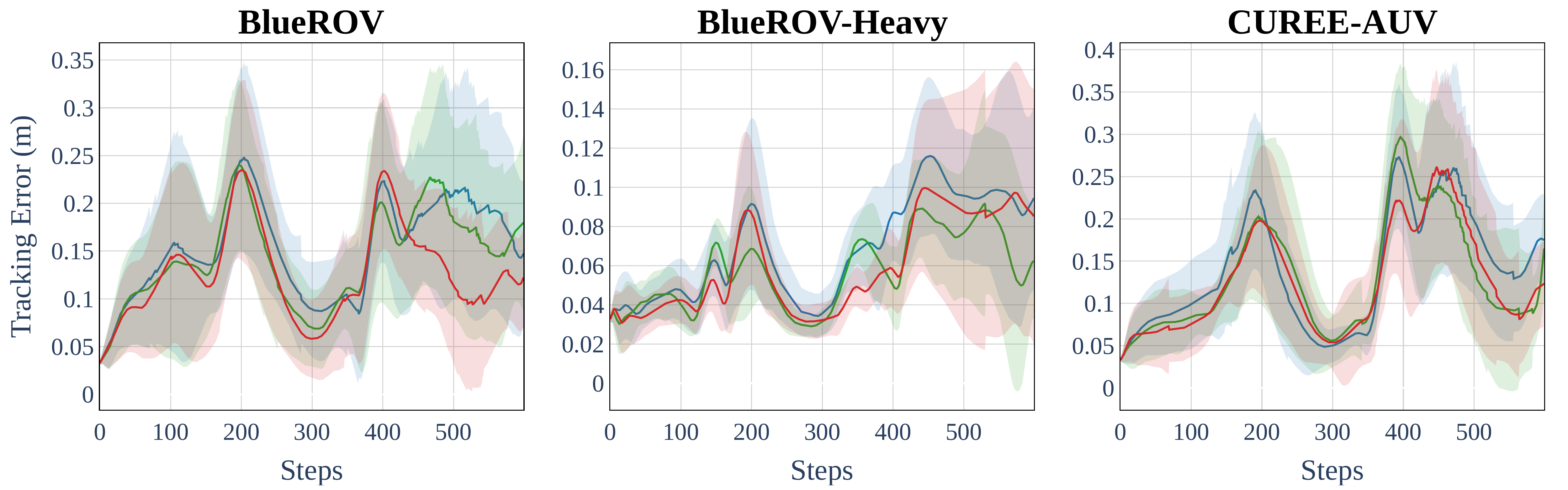}
        \caption{Track-Lemniscate}
    \end{subfigure}\hfill
    \begin{subfigure}{0.49\linewidth}
        \centering
        \includegraphics[width=\linewidth]{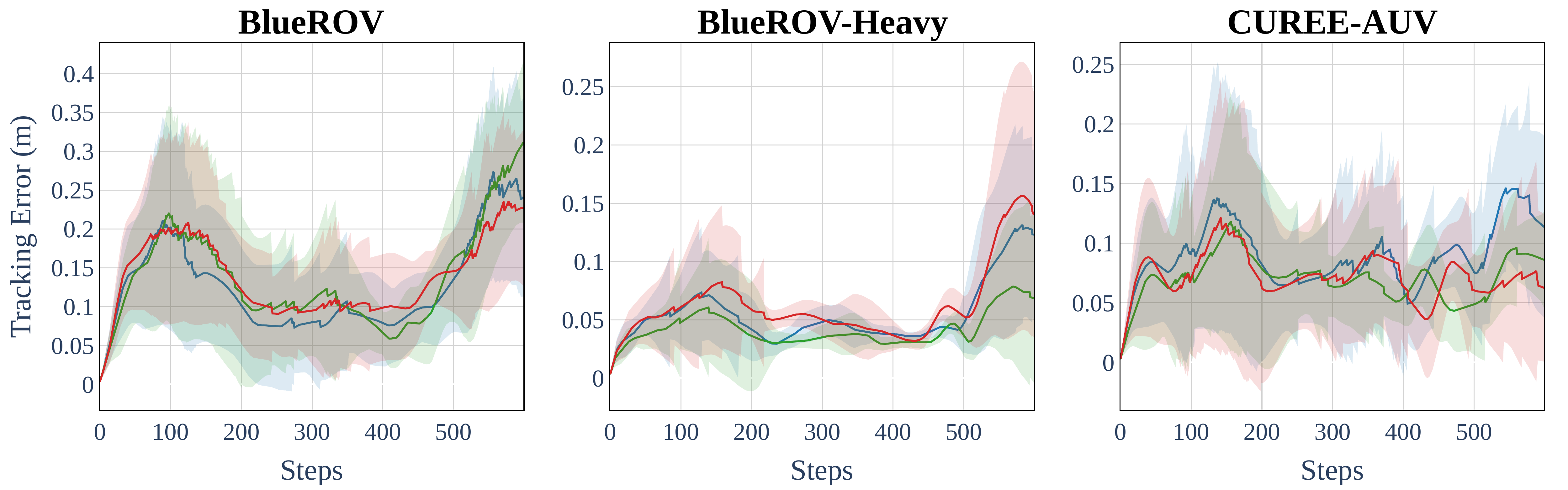}
        \caption{Track-Circle}
    \end{subfigure}\\[2mm]
    \begin{subfigure}{0.49\linewidth}
        \centering
        \includegraphics[width=\linewidth]{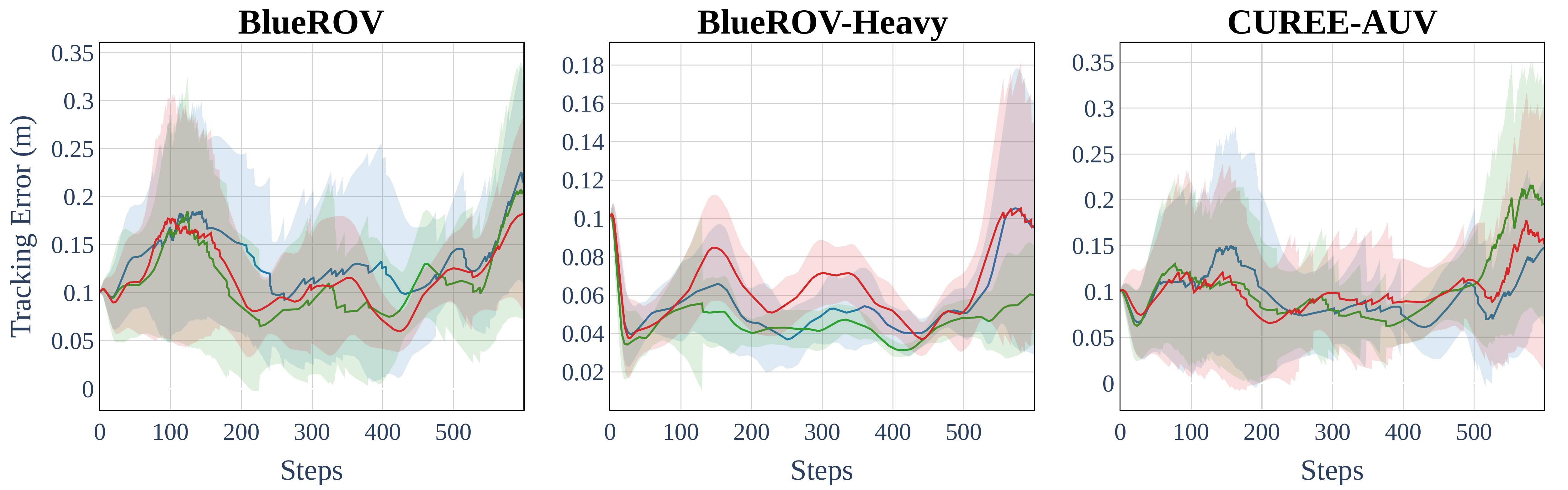}
        \caption{Track-Spiral}
    \end{subfigure}\\
    \vspace{1mm}
    \includegraphics[width=0.6\linewidth]{figures/legend.pdf}
    \caption{Tracking error over an evaluation episode, per vehicle (columns); each curve is the mean over 100 seeds.}
    \label{fig:tracking}
\end{figure}

\subsection{RQ4: Task Quality and Trajectories}
The average-power and smoothness gains largely do not cost task quality, and training return stays comparable across methods (Fig.~\ref{fig:curves}). On the strongly actuated BlueROV-Heavy, PPO-Lag holds $95$--$100\%$ success and low tracking error across all tracking tasks while drawing the least average power; BlueROV-Heavy spiral, for instance, reaches $100\%$ success at $0.127$\,m error. On the lightly actuated BlueROV, where the curved references are hard for every method, PPO-Lag gives the best success on the lemniscate ($92\%$ versus $71$--$72\%$) and spiral ($59\%$ versus $47$--$52\%$), and on the circle it improves over PPO ($45\%$ versus $35\%$, second to PPO-Energy's $48\%$) while attaining the lowest tracking error ($1.53$\,m versus PPO's $3.03$\,m, $p<0.01$). For hover, time-to-goal is comparable to or better than the baselines: PPO-Lag reaches the setpoint fastest on CUREE-AUV ($133.5$ steps vs.\ $138.5$ for PPO and $156.2$ for PPO-Energy, $p<0.01$) and is within a few steps of PPO on the other two vehicles. The tracking-error rollouts in Fig.~\ref{fig:tracking} confirm that PPO-Lag follows the reference as closely as the baselines.

The clearest task-quality trade-off is on the two highest-curvature CUREE-AUV tracking tasks: PPO-Lag reaches $69\%$ success on the circle versus $80\%$ for PPO, and $66\%$ on the spiral versus $74\%$, while cutting average power by $18\%$ and $16\%$ respectively. Under this vehicle's tight budget the constraint withdraws some of the high-authority corrective actuation the baseline spends to track precisely; the budget $d$ thus surfaces the endurance-versus-precision trade-off as an explicit, adjustable knob in watts rather than burying it in a reward weight. Tellingly, the energy-reward PPO-Energy is worse than both here (e.g., $58\%$ success on the circle at higher average power), confirming that an unconstrained penalty neither reduces average power nor protects accuracy.

Fig.~\ref{fig:traj} shows the executed 3-D trajectories for all three vehicles and four tasks, with the three policies overlaid. For tracking, every policy follows the reference closely; for hover, the trajectories converge from randomized initial poses to the central setpoint. The visual similarity of the trajectories across policies, despite the large average-power differences in Table~\ref{tab:performance}, is the qualitative counterpart of the main finding: comparable motion is achieved at substantially lower average power.

\begin{figure}
    \centering
    \makebox[0.03\linewidth]{}\hfill
    \makebox[0.30\linewidth][c]{\scriptsize\textbf{BlueROV}}\hfill
    \makebox[0.30\linewidth][c]{\scriptsize\textbf{BlueROV-Heavy}}\hfill
    \makebox[0.30\linewidth][c]{\scriptsize\textbf{CUREE-AUV}}\\[3pt]
    \begin{minipage}[c]{0.03\linewidth}\centering\rotatebox{90}{\small\textbf{Hover}}\end{minipage}\hfill
    \begin{minipage}[c]{0.32\linewidth}\centering\includegraphics[width=\linewidth]{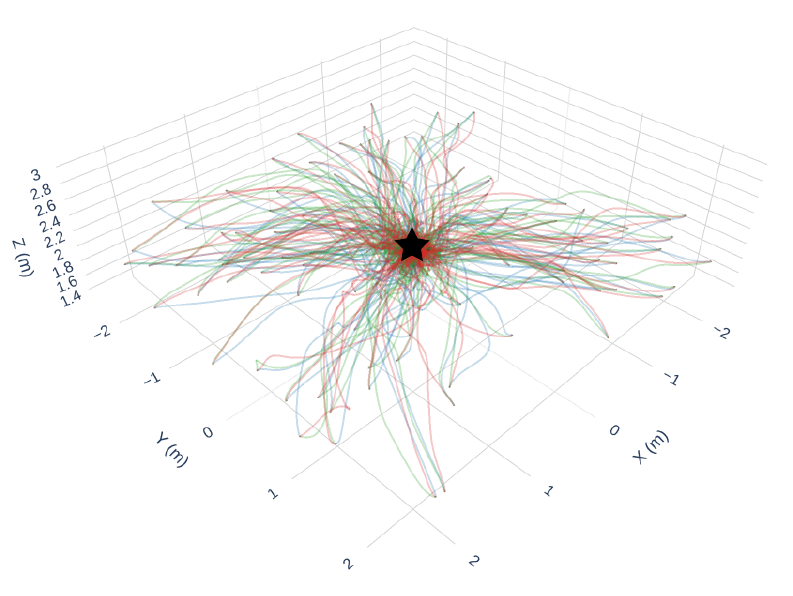}\end{minipage}\hfill
    \begin{minipage}[c]{0.32\linewidth}\centering\includegraphics[width=\linewidth]{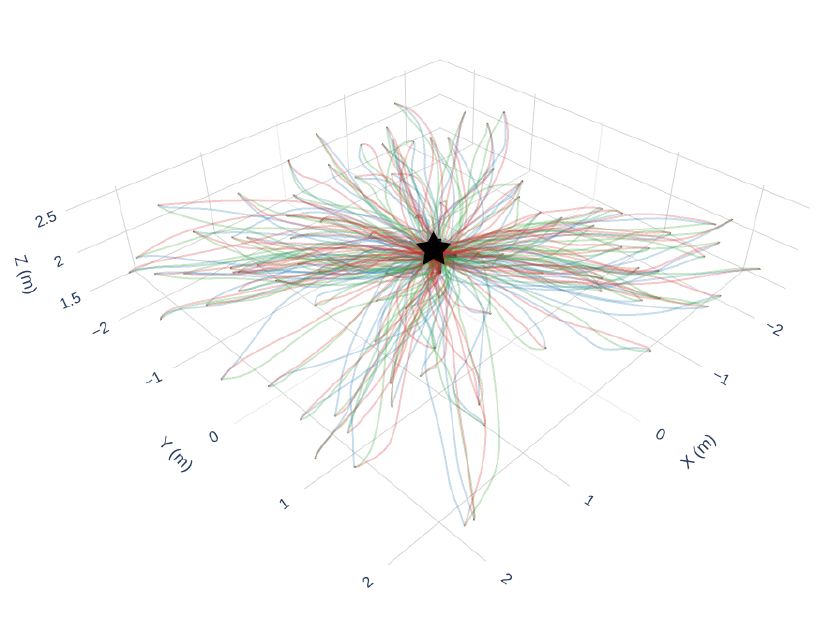}\end{minipage}\hfill
    \begin{minipage}[c]{0.32\linewidth}\centering\includegraphics[width=\linewidth]{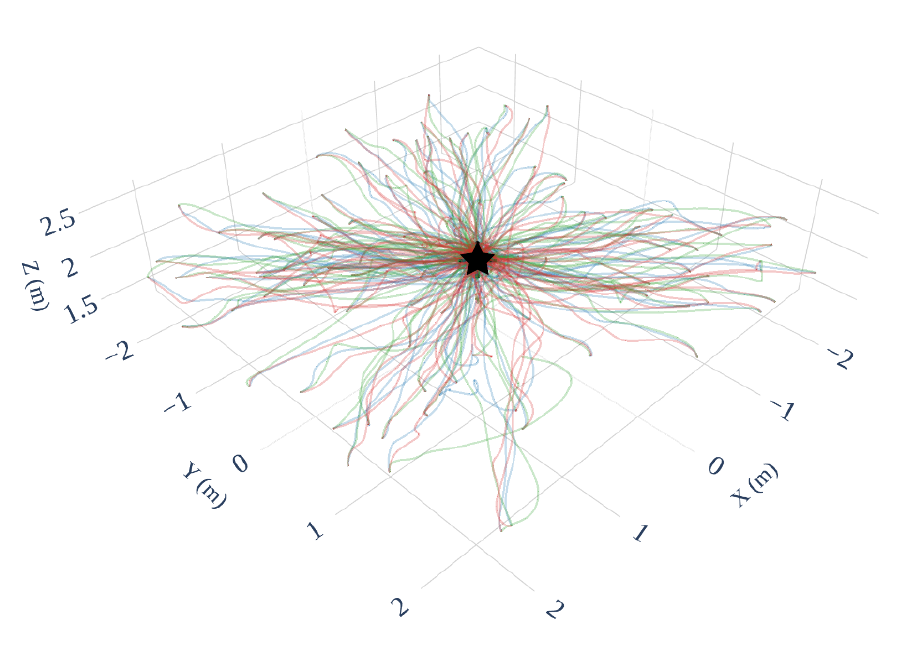}\end{minipage}\\[3pt]
    \begin{minipage}[c]{0.03\linewidth}\centering\rotatebox{90}{\small\textbf{Track-Lemniscate}}\end{minipage}\hfill
    \begin{minipage}[c]{0.32\linewidth}\centering\includegraphics[width=\linewidth]{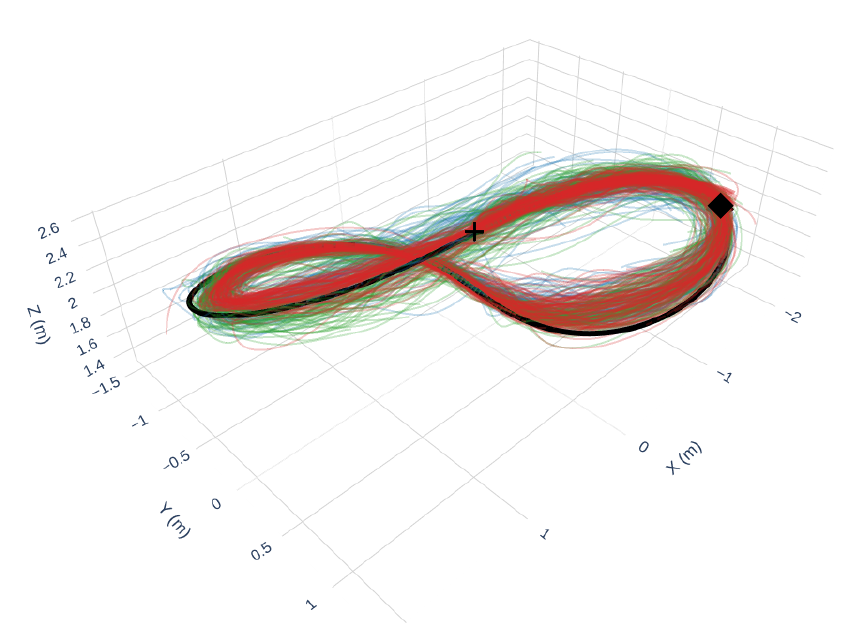}\end{minipage}\hfill
    \begin{minipage}[c]{0.32\linewidth}\centering\includegraphics[width=\linewidth]{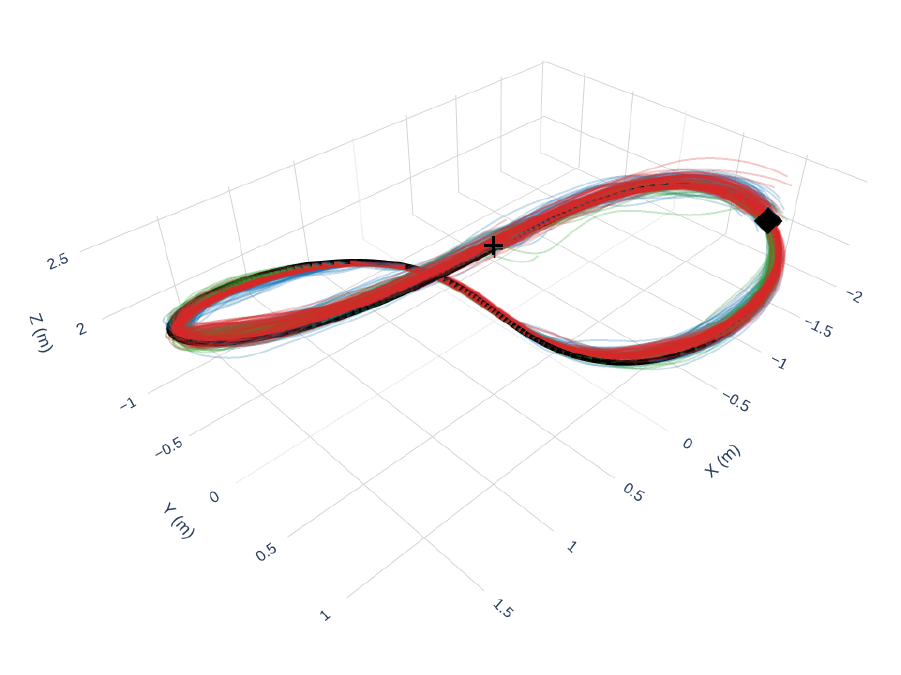}\end{minipage}\hfill
    \begin{minipage}[c]{0.32\linewidth}\centering\includegraphics[width=\linewidth]{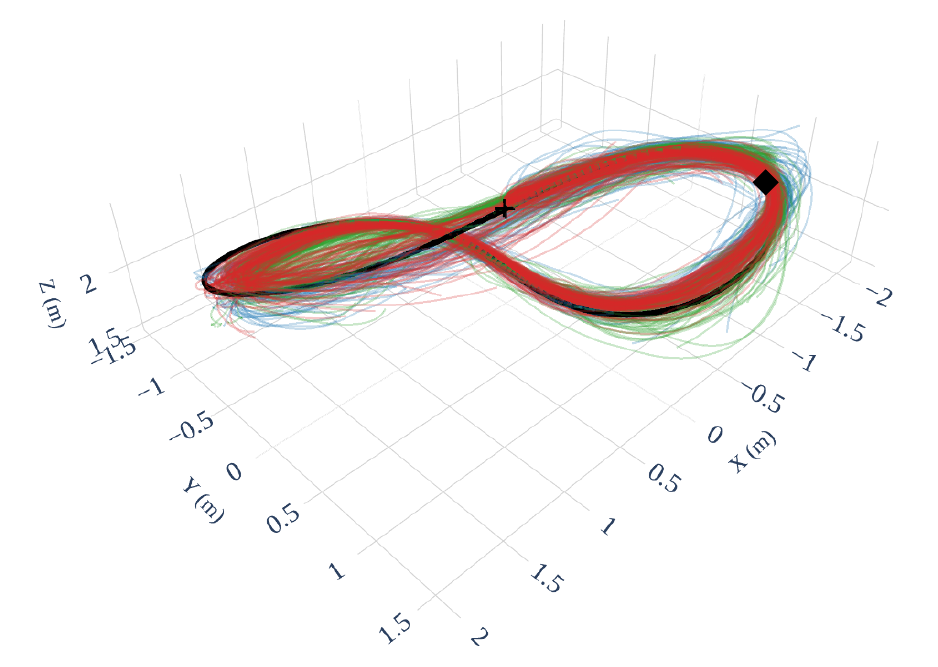}\end{minipage}\\[3pt]
    \begin{minipage}[c]{0.03\linewidth}\centering\rotatebox{90}{\small\textbf{Track-Circle}}\end{minipage}\hfill
    \begin{minipage}[c]{0.32\linewidth}\centering\includegraphics[width=\linewidth]{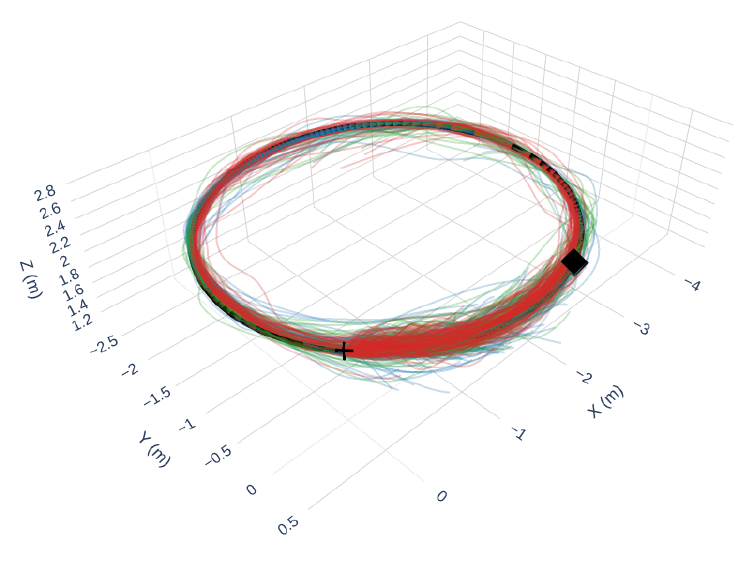}\end{minipage}\hfill
    \begin{minipage}[c]{0.32\linewidth}\centering\includegraphics[width=\linewidth]{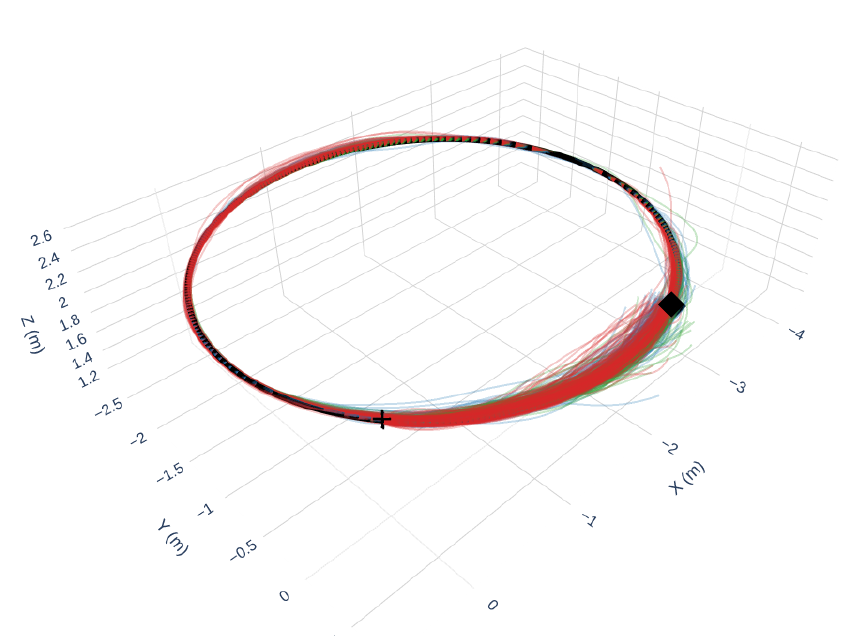}\end{minipage}\hfill
    \begin{minipage}[c]{0.32\linewidth}\centering\includegraphics[width=\linewidth]{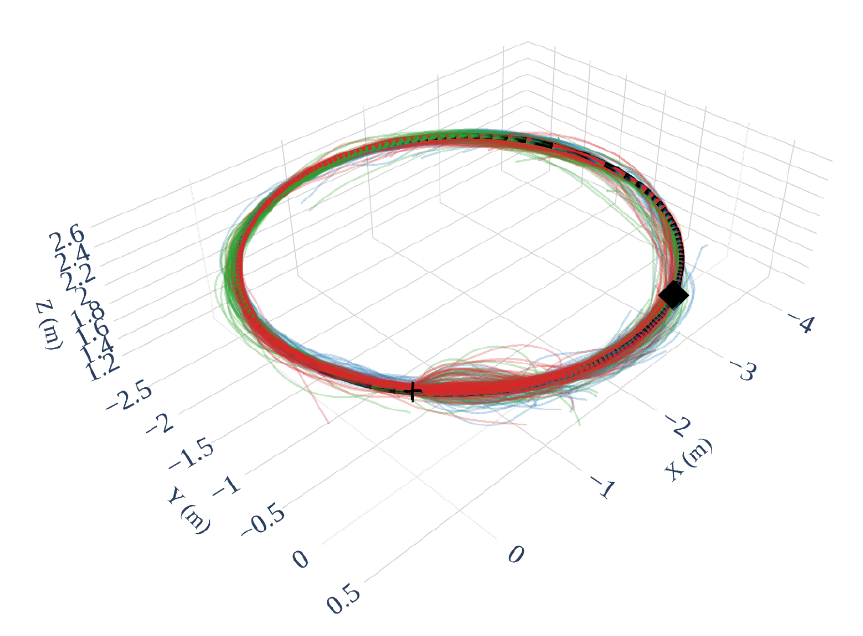}\end{minipage}\\[3pt]
    \begin{minipage}[c]{0.03\linewidth}\centering\rotatebox{90}{\small\textbf{Track-Spiral}}\end{minipage}\hfill
    \begin{minipage}[c]{0.32\linewidth}\centering\includegraphics[width=\linewidth]{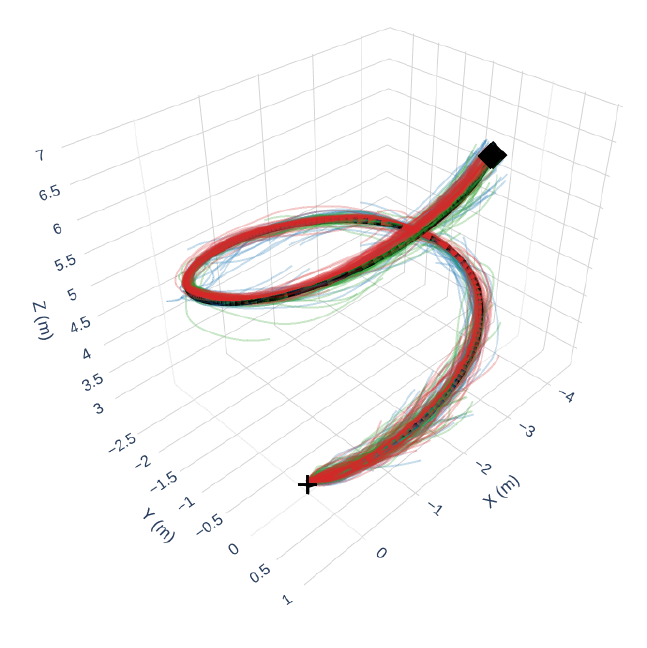}\end{minipage}\hfill
    \begin{minipage}[c]{0.32\linewidth}\centering\includegraphics[width=\linewidth]{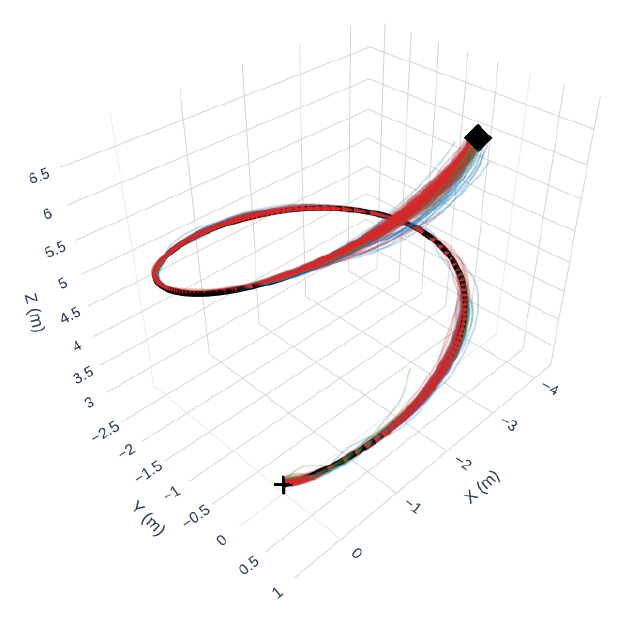}\end{minipage}\hfill
    \begin{minipage}[c]{0.32\linewidth}\centering\includegraphics[width=\linewidth]{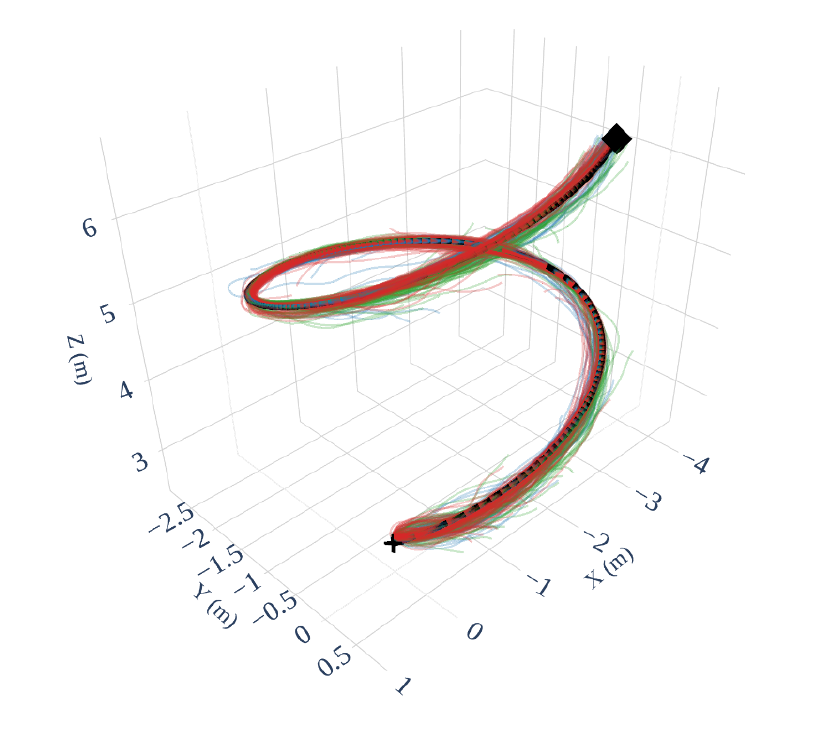}\end{minipage}
    
    \vspace{2mm}
    \includegraphics[width=0.8\linewidth]{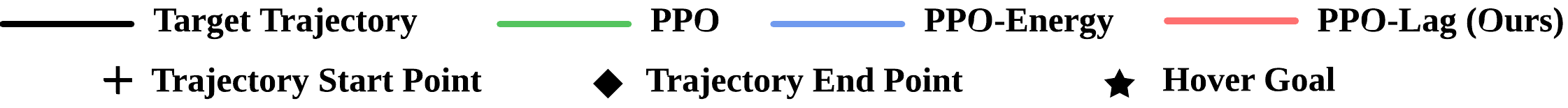}
    \caption{Executed 3-D trajectories for every vehicle (columns) and task (rows), with the three policies overlaid.}
    \label{fig:traj}
\end{figure}

\subsection{RQ5: On-board Resource Cost}
Finally, we profile the controller's on-board cost. Because the energy treatment enters only during training, the three methods deploy the identical actor network, so their inference-time footprint is the same by construction; Table~\ref{tab:resource_usage} and Fig.~\ref{fig:orin} report it for the deployed PPO-Lag policy on an embedded NVIDIA Jetson Orin Nano (8\,GB), per task and vehicle. The cost is modest and uniform across settings: $143$k--$146$k parameters (fixed by the vehicle's state and action dimensions), sub-millisecond inference ($0.93$--$0.96$\,ms), near-saturated single-core central processing unit (CPU) usage ($\sim$99--100\%), $\sim$640--646\,MB random-access memory (RAM), $\sim$105\,MB GPU memory, and $\sim$5.1--5.6\,W board power. The controller therefore runs comfortably on resource-constrained hardware, and the average-power reductions reported above come at no on-board compute cost.

\begin{figure}[tb]
    \centering
    \begin{subfigure}{0.49\linewidth}
        \centering
        \includegraphics[width=\linewidth, height=2.5cm]{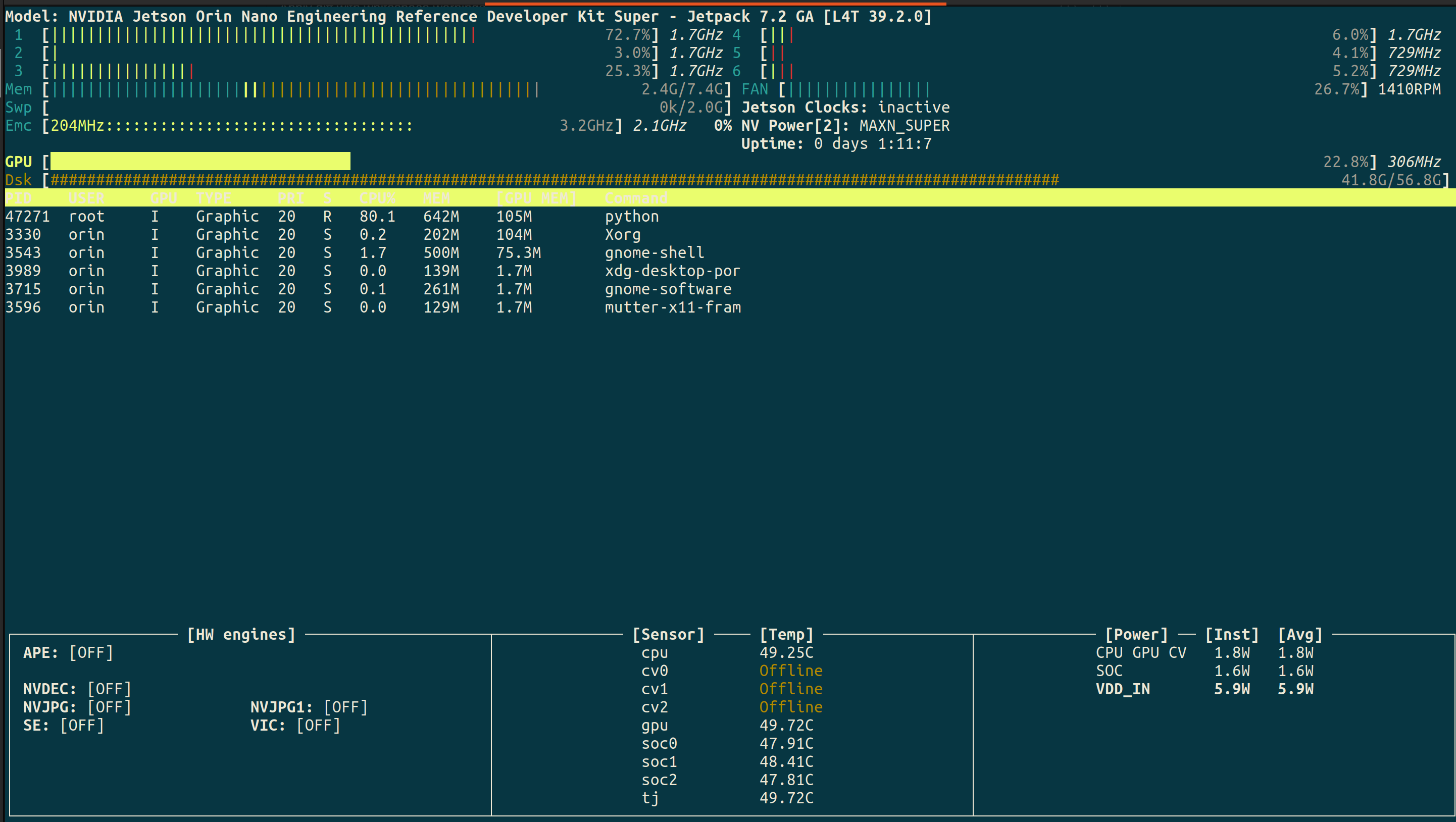}
        \label{fig:all}
    \end{subfigure}
    \hfill
    \begin{subfigure}{0.49\linewidth}
        \centering
        \includegraphics[width=\linewidth, height=2.5cm]{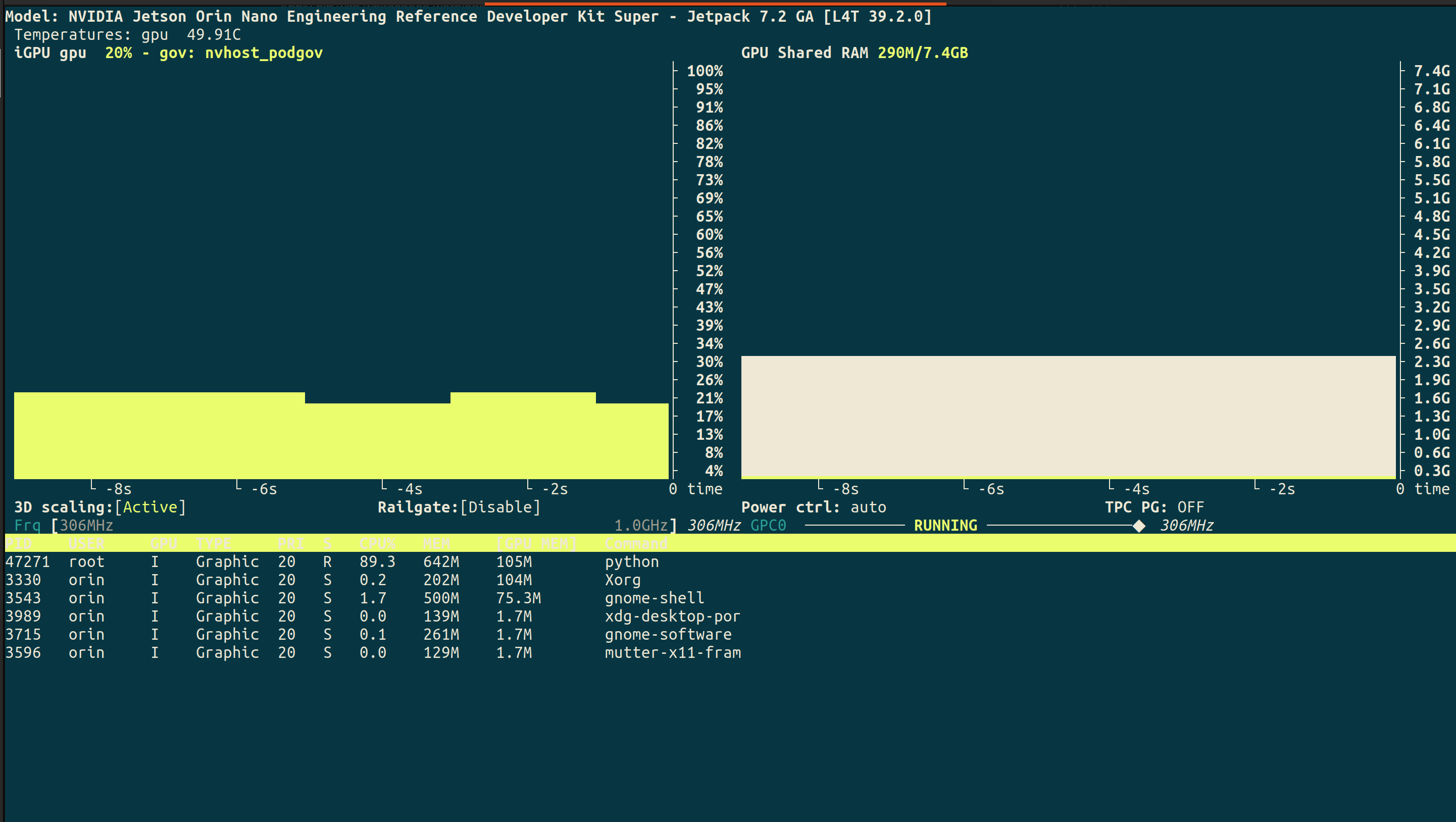}
        \label{fig:gpu}
    \end{subfigure}

    \vspace{2pt}
    \begin{subfigure}{0.49\linewidth}
        \centering
        \includegraphics[width=\linewidth, height=2.5cm]{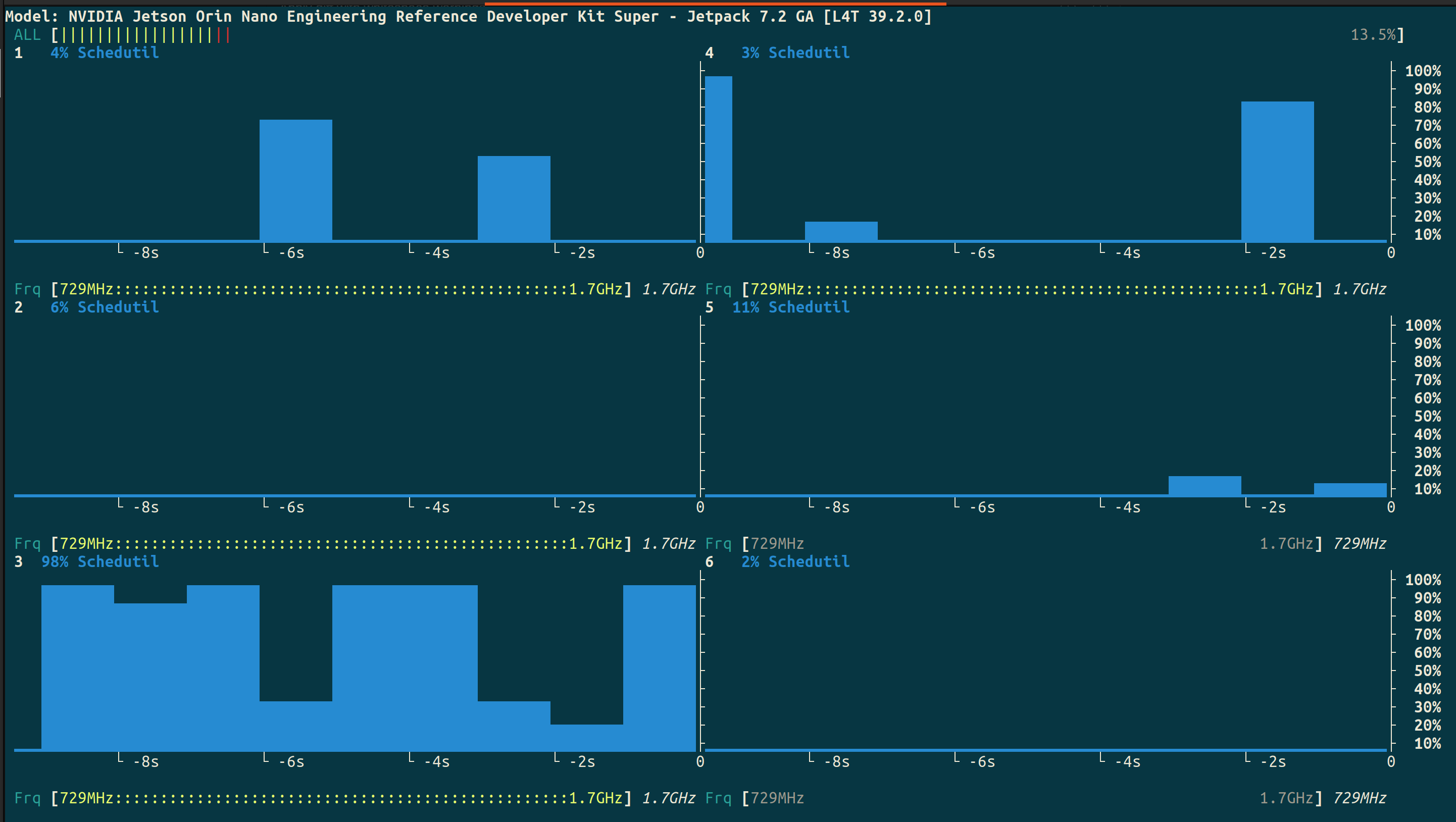}
        \label{fig:cpu}
    \end{subfigure}
    \hfill
    \begin{subfigure}{0.49\linewidth}
        \centering
        \includegraphics[width=\linewidth,trim={0pt 300pt 0pt 350pt}, clip, height=2.5cm]{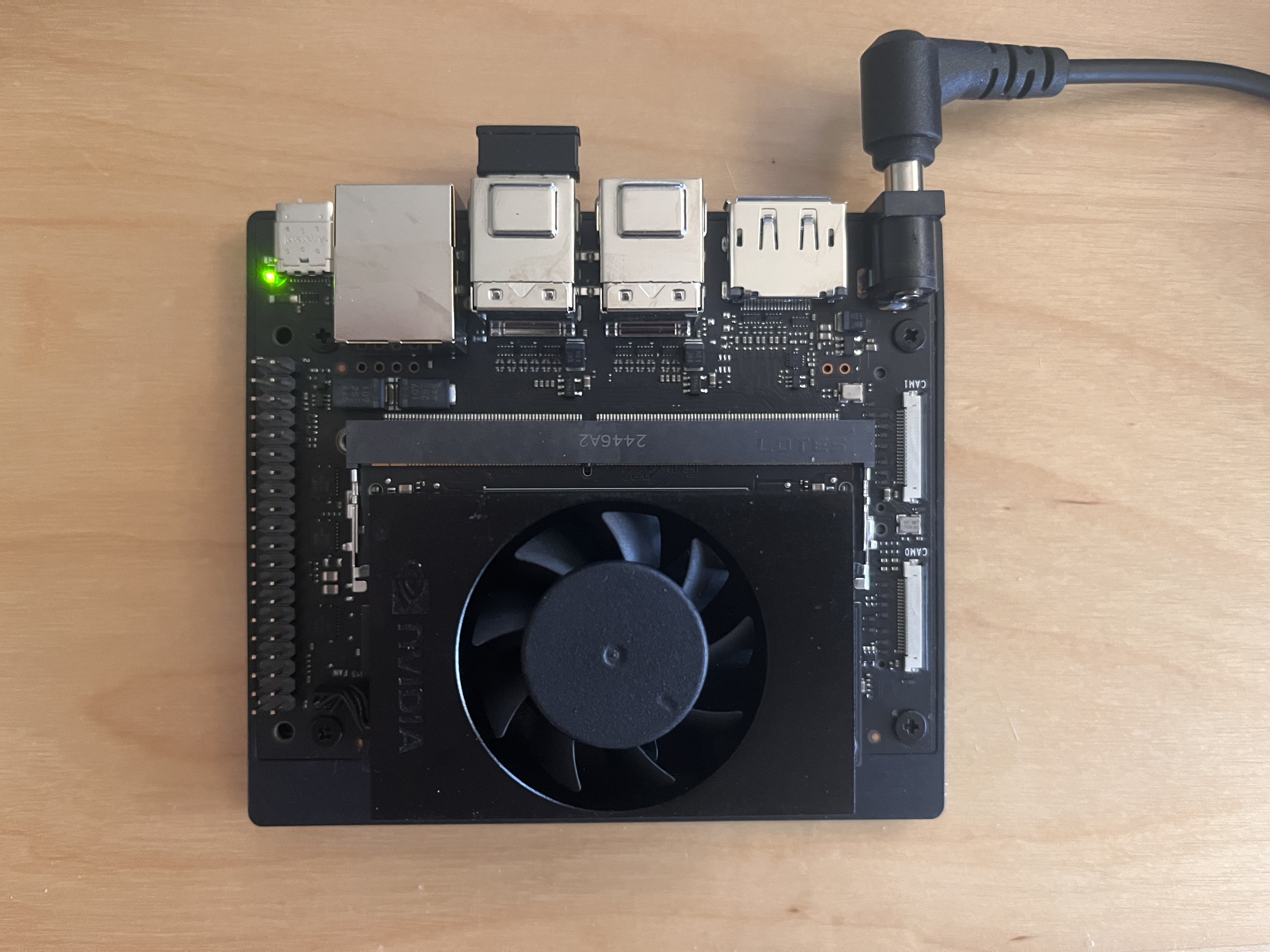}
        \label{fig:ram}
    \end{subfigure}
    \caption{On-board resource utilization on a Jetson Orin Nano: overall load (top-left), GPU (top-right), CPU (bottom-left), and hardware (bottom-right).}
    \label{fig:orin}
\end{figure}

\begin{table*}
\centering
\scriptsize
\setlength{\tabcolsep}{3pt}
\caption{On-board resource cost of the deployed PPO-Lag policy on a Jetson Orin Nano, per task and vehicle; the three policies share the same network, so this footprint is identical for all.}
\label{tab:resource_usage}
\resizebox{\textwidth}{!}{%
\begin{tabular}{llcccccccccccc}
\toprule
\multirow{2}{*}{Task} & \multirow{2}{*}{Robot} & \multirow{2}{*}{Policy Params} & \multirow{2}{*}{Inf. (ms)} & \multicolumn{2}{c}{CPU Util. (\%)} & \multicolumn{2}{c}{RAM (MB)} & \multicolumn{2}{c}{GPU Util. (\%)} & \multicolumn{2}{c}{GPU Mem (MB)} & \multicolumn{2}{c}{Power (W)} \\

\cmidrule(lr){5-6}\cmidrule(lr){7-8}\cmidrule(lr){9-10}\cmidrule(lr){11-12}\cmidrule(lr){13-14}
&&&& Mean & Max & Mean & Max & Mean & Max & Mean & Max& Mean & Max \\

\midrule
Hover & BlueROV  & 143,116 & 0.9395 & 99.6 & 109.4 & 646.0 & 647.2 & 18.0 & 25.1 & 105.1 & 105.6 & 5.20 & 5.93 \\
Hover & BlueROV-Heavy  & 144,144 & 0.9477 & 98.8 & 109.3 & 646.6 & 647.2 & 23.3 & 69.9 & 105.0 & 105.6 & 5.15 & 5.81 \\
Hover & CUREE-AUV  & 143,116 & 0.9338 & 100.1 & 109.4 & 644.7 & 645.2 & 16.1 & 27.5 & 104.9 & 105.6 & 5.15 & 5.85 \\
\midrule
Track-Lemniscate & BlueROV  & 144,652 & 0.9336 & 99.9 & 109.5 & 644.5 & 645.1 & 18.0 & 28.3 & 105.1 & 105.6 & 5.20 & 6.01 \\
Track-Lemniscate & BlueROV-Heavy & 145,680 & 0.9310 & 99.9 & 109.5 & 644.6 & 645.3 & 16.6 & 24.5 & 105.0 & 105.6 & 5.24 & 5.85 \\
Track-Lemniscate & CUREE-AUV  & 144,652 & 0.9533 & 99.2 & 109.4 & 644.4 & 645.1 & 14.7 & 23.9 & 104.9 & 105.6 & 5.14 & 5.85 \\
\midrule
Track-Circle & BlueROV & 144,652 & 0.9472 & 99.8 & 109.6 & 644.3 & 644.8 & 25.1 & 53.5 & 105.1 & 105.6 & 5.53 & 6.80 \\
Track-Circle & BlueROV-Heavy  & 145,680 & 0.9319 & 99.5 & 109.3 & 644.7 & 645.3 & 17.9 & 27.6 & 105.1 & 105.6 & 5.21 & 6.01 \\
Track-Circle & CUREE-AUV  & 144,652 & 0.9402 & 99.6 & 109.5 & 644.2 & 644.6 & 16.7 & 24.7 & 105.0 & 105.6 & 5.31 & 6.01 \\
\midrule
Track-Spiral & BlueROV  & 144,652 & 0.9525 & 99.7 & 109.5 & 645.4 & 645.9 & 16.9 & 25.1 & 105.1 & 105.6 & 5.21 & 5.85 \\
Track-Spiral & BlueROV-Heavy  & 145,680 & 0.9388 & 99.4 & 109.4 & 645.4 & 645.9 & 20.0 & 30.8 & 105.1 & 105.6 & 5.62 & 5.93 \\
Track-Spiral & CUREE-AUV  & 144,652 & 0.9382 & 99.8 & 109.3 & 645.4 & 646.0 & 17.3 & 26.4 & 105.0 & 105.6 & 5.38 & 6.45 \\
\bottomrule
\end{tabular}%
}
\end{table*}

\section{Conclusion}
\label{sec:conclusion}
This paper studies energy-efficient underwater vehicle control through constrained reinforcement learning. By formulating control as a CMDP with an explicit average-power budget and solving it with a PPO-Lagrangian algorithm, we obtain a controller whose average operating power is set by an explicit, declarable, budget through an adaptive dual variable rather than a hand-tuned reward weight, supported by a short analysis showing that the dual update is well posed and that the average budget overshoot vanishes in the idealized setting. Across three vehicles and four tasks in MarineGym, the constrained policy attains the lowest average power in all twelve settings (14--32\% reduction over the task-only baseline, largest on BlueROV-Heavy) and the smoothest actuation in eleven of twelve, while largely preserving task success and keeping tracking accuracy comparable to the baselines; a fixed-weight action-effort reward, by contrast, is inconsistent and at times increases average power. Treating average power as a constraint thus provides a tuning-free route to energy-efficient underwater control that needs no per-vehicle, per-task weight search. Future work includes deployment on physical vehicles to assess the sim-to-real gap, joint constraints on average power and additional resources (e.g., thermal or thrust limits), online adaptation of the average-power budget to mission state, and extension to multi-vehicle and disturbance-rich settings.

\bibliographystyle{IEEEtran}
\bibliography{IEEEabrv,Bibliography}

\vfill

\end{document}